\documentclass{article}

\PassOptionsToPackage{numbers, compress}{natbib}
\usepackage[preprint]{neurips_2023}

\usepackage[utf8]{inputenc} % allow utf-8 input
\usepackage[T1]{fontenc}    % use 8-bit T1 fonts
\usepackage{hyperref}       % hyperlinks
\usepackage{url}            % simple URL typesetting
\usepackage{booktabs}       % professional-quality tables
\usepackage{amsfonts}       % blackboard math symbols
\usepackage{nicefrac}       % compact symbols for 1/2, etc.
\usepackage{microtype}      % microtypography
\usepackage{xcolor}         % colors
\usepackage{subfig}
\usepackage{wrapfig}

\hypersetup{
	colorlinks   = true, %Colours links instead of ugly boxes
	urlcolor     = blue, %Colour for external hyperlinks
	linkcolor    = blue, %Colour of internal links
	citecolor   = blue %Colour of citations
}

\usepackage{amsmath}
\usepackage{amssymb}
\usepackage{mathtools}
\usepackage{amsthm}

\usepackage{bm,bbm}
\usepackage{xspace}
\usepackage{multirow}
\usepackage{cancel}
\usepackage[normalem]{ulem}
\usepackage{enumitem}
\usepackage{caption}
\usepackage{booktabs}
\usepackage{algorithm}
\usepackage[noend]{algpseudocode}

\theoremstyle{plain}
\newtheorem{theorem}{Theorem}[section]

\newtheorem{lemma}[theorem]{Lemma}

\theoremstyle{definition}

\newtheorem{assumption}[theorem]{Assumption}
\newtheorem{remark}[theorem]{Remark}

\newtheorem{claim}{Claim}
\usepackage{etoolbox}
\AtEndEnvironment{proof}{\setcounter{claim}{0}}

\usepackage{letltxmacro}
\LetLtxMacro{\originaleqref}{\eqref}
\renewcommand{\eqref}{Eq.\originaleqref}

\DeclareMathOperator*{\argmax}{arg\,max}
\DeclareMathOperator*{\argmin}{arg\,min}
\DeclareMathOperator*{\KL}{KL}
\DeclareMathOperator*{\kl}{kl}

\DeclarePairedDelimiter\abs{\lvert}{\rvert}%
\DeclarePairedDelimiter{\ceil}{\lceil}{\rceil}

\newcommand{\E}{\mathbb{E}}

\newcommand{\1}[1]{\mathbbm{1}{\left\{#1\right\}}}
\renewcommand{\P}{\mathbb{P}}
\newcommand{\upbra}[1]{^{(#1)}}

\renewcommand{\ge}{\geqslant}
\renewcommand{\le}{\leqslant}

\newcommand{\BAIC}{\texttt{BAI}\xspace}
\newcommand{\MFMAB}{\texttt{MF-MAB}\xspace}
\newcommand{\explore}[1]{\textsc{Explore-#1}\xspace}

\usepackage[colorinlistoftodos,prependcaption,textsize=tiny,textwidth=0.15\columnwidth]{todonotes}

\newcommand{\compilehidecomments}{false} %HIDE comments if true
\ifthenelse{ \equal{\compilehidecomments}{false} }{%
	\newcommand{\xuchuang}[1]{{\color{red}  [\text{Xuchuang:} #1]}}
	\newcommand{\wei}[1]{{\color{blue}  [\text{Wei:} #1]}}
 \newcommand{\qingyun}[1]{{\color{violet}  [\text{Qingyun:} #1]}}
}{
	\newcommand{\xuchuang}[1]{}
	\newcommand{\wei}[1]{}
 \newcommand{\qingyun}[1]{}
}

\newcommand{\compilefullversion}{true} %Compile Full Version
\ifthenelse{\equal{\compilefullversion}{false}}{%
	\newcommand{\OnlyInFull}[1]{}
	\newcommand{\OnlyInShort}[1]{#1}
}{
	\newcommand{\OnlyInFull}[1]{#1}%
	\newcommand{\OnlyInShort}[1]{}%
}
\title{Multi-Fidelity Multi-Armed Bandits Revisited}

\author{%
  Xuchuang Wang \\
  The Chinese University of Hong Kong\\
  \texttt{xcwang@cse.cuhk.edu.hk} \\
  \And
  Qingyun Wu \\
  Pennsylvania State University \\
  \texttt{qingyun.wu@psu.edu} \\
  \AND
  Wei Chen \\
  Microsoft Research \\
  \texttt{weic@microsoft.com} \\
  \And
  John C.S. Lui \\
  The Chinese University of Hong Kong \\
  \texttt{cslui@cse.cuhk.edu.hk} \\
}

\begin{document}

\maketitle

\begin{abstract}
  We study the multi-fidelity multi-armed bandit (\MFMAB), an extension of the canonical multi-armed bandit (MAB) problem.
  \MFMAB allows each arm to be pulled with different costs (fidelities) and observation accuracy.
  We study both the best arm identification with fixed confidence (\BAIC) and the regret minimization objectives.
  For \BAIC, we present (a) a cost complexity lower bound, (b) an algorithmic framework with two alternative fidelity selection procedures,
  and (c) both procedures' cost complexity upper bounds.
  From both cost complexity bounds of \MFMAB,
  one can recover the standard sample complexity bounds of the classic (single-fidelity) MAB.
  For regret minimization of \MFMAB, we propose a new regret definition, prove its problem-independent regret lower bound \(\Omega(K^{1/3}\Lambda^{2/3})\) and problem-dependent lower bound \(\Omega(K\log \Lambda)\), where \(K\) is the number of arms and \(\Lambda\) is the decision budget in terms of cost, and devise an elimination-based algorithm whose worst-cost regret upper bound matches its corresponding lower bound up to some logarithmic terms and, whose problem-dependent bound matches its corresponding lower bound in terms of \(\Lambda\).
\end{abstract}

% !TeX root = ..\mf-mab.tex
\section{Introduction}

The multi-armed bandits (MAB) problem was first introduced in the seminal work by~\citet{lai1985asymptotically}
and was extensively studied (ref.~\citet{bubeck2012regret,lattimore2020bandit}).
In stochastic MAB, the decision maker repeatedly pulls arms among a set of \(K\in\mathbb{N}^+\) arms and observes rewards drawn from unknown distributions of the pulled arms.
The initial objective
% \qingyun{Change "criterion" to "objective"?}\xuchuang{Yes, objective is better, thanks.} 
of MAB is \emph{regret minimization}~\citep{lai1985asymptotically,auer2010ucb}, where the regret is the cumulative differences between the rewards from pulling the optimal arm and that from the concerned algorithm's arm pulling strategy.
This is a fundamental sequential decision making framework for studying the exploration-exploitation trade-off, where one needs to balance between optimistically exploring arms with high uncertainty in reward (exploration) and myopically exploiting arms with high empirical reward average (exploitation).
Another task of MAB is the {best arm identification} (\BAIC,  a.k.a., pure exploration), later introduced by~\citet{even2002pac,even2006action,mannor2004sample}.
Best arm identification aims to find the best arm \emph{without} considering the cumulative regret (reward) during the learning process.
In this paper, we focus on the fixed confidence case, with the goal of identifying the best arm with a fixed confidence with as few number of decision rounds as possible.
%
%, called \emph{best arm identification with fixed confidence},
%another is to identify the best arm within a fixed budget with a confidence as high as possible, called \emph{best arm identification with fixed budget}.
%Since we do not study the fixed budget case in this paper, thereafter, we refer the fixed confidence case as best arm identification (\BAIC) for brevity.

In this paper, we investigate a generalized MAB model with a wide range of real-world applications---the \emph{multi-fidelity multi-armed bandit} (\MFMAB)---introduced by~\citet{kandasamy2016multi}.
We study both the regret minimization and best arm identification objectives under the \MFMAB model.
% \qingyun{Need to mention the two objectives introduced in the preceding paragraph to make the transition more natural.}
The \MFMAB model introduces flexibility for exploring arms: instead of squarely pulling an arm to observe its reward sample as in MAB,
\MFMAB offers \(M\in\mathbb{N}^+\) different accesses to explore an arm, called \(M\) fidelities.
Pulling an arm \(k\in\mathcal{K}\coloneqq \{1,\dots,K\}\) at fidelity \(m\in\mathcal{M}\coloneqq \{1,\dots,M\}\), the decision maker pays a cost of \(\lambda\upbra{m}\) and observes a reward \(X_{k}\upbra{m}\) whose mean \(\mu_k\upbra{m}\) is not too far away from the arm's true reward mean \(\mu_k\).
More formally, \(\abs{\mu_k\upbra{m} - \mu_k}\le \zeta\upbra{m}\) where \(\zeta\upbra{m} \ge 0\) is the error upper bound at fidelity \(m\).
% We label these \(M\) fidelities from \(1\) to \(M\) such that their costs \(\lambda\upbra{m}\) increase w.r.t. \(m\), and we take the reward mean at the highest fidelity \(M\) as the arm's true reward mean, i.e., \(\mu_k\upbra{M}=\mu_k\) and $\zeta\upbra{M} = 0$.
The formal definition of \MFMAB is presented in Section~\ref{sec:model}.

\subsection{Contributions}
In Section~\ref{sec:bai},
% \todo{If we need more space, the contribution subsection could be shorter.}
we first look into the best arm identification with fixed confidence (\BAIC) objective in \MFMAB, which has wide applications in hyperparameter optimization~\citep[\S1.4]{hutter2019automated}
and neural architecture search (NAS)~\cite{elsken2019neural}.
In the context of NAS, each arm of \MFMAB corresponds to one configuration of neural network architecture,
and the different fidelities of pulling this arm correspond to training the neural network of this configuration with different sample sizes of training data (see details in Section~\ref{subsec:application_bai}).
As in \MFMAB pulling an arm at different fidelities has different costs, we consider the total cost needed by an algorithm and refer to it as \emph{cost complexity} (a generalization of sample complexity).
We first derive a cost complexity lower bound \(
\Omega ( \sum_{k\in\mathcal{K}}\min_{m:\Delta_k\upbra{m} > 0}({\lambda\upbra{m}}/{(\Delta_k\upbra{m})^{2}}) \log {\delta}^{-1} ),
\)
where the \(\Delta_k\upbra{m}\) is the reward gap of arm \(k\) at fidelity \(m\) (defined in~\eqref{eq:reward-gap}) and \(\delta\) is the confidence parameter.
% This lower bound suggests that for each arm \(k\), there exists an optimal (or most efficient) fidelity for exploring it in order to identify whether this arm is optimal or not,
% % \qingyun{what does "explore" mean here? Be more accurate?}\wei{tried to add some words.}
% which highlights the advantage of the flexibility of multi-fidelity over the classic single-fidelity case.
We then devise a Lower-Upper Confidence Bound (LUCB) algorithmic framework with two different fidelity selection procedures and prove both procedures' cost complexity upper bounds.
% , both of which can recover the classic MAB's sample complexity upper bound when \MFMAB is reduced to (single-fidelity) MAB.
% \wei{We need to say something on why we need three different procedures. Usually one is enough, if we can identify the best one. Providing three means there is some tradeoff, and we are not sure which one is
% 	the best. We need to say something here. The following is my try.}
% Each procedure emphasizes on one aspect in selecting proper fidelities for arm exploration:
% The first procedure focuses on selecting the fidelities from low to high in a sequential order, considering the increasing
% cost in playing higher fidelities, but does not attempt to identify an optimal fidelity for each arm as suggested by the lower bound.
The first procedure focuses on finding the optimal fidelity suggested by the minimization in lower bound via an upper confidence bound (UCB) algorithm, which may pay high additional costs during searching the optimal fidelities.
The second procedure, instead of identifying the optimal fidelity, seeks for good fidelities that are at least half as good as the optimal ones so as to reduce the cost for identifying the optimal fidelity while still enjoying fair theoretical performance.
  % removes the strong assumption of the second procedure but may not be able to identify the optimal fidelity as expected.
  % \wei{How to compare these cost complexity results? Is there some tradeoffs that can be derived and explained? Do we have any conclusion on when to use which procedure?}
  {Note that despite several closely related works (discussed in Section~\ref{sec:related-works}), the \BAIC task under \MFMAB model has not been studied in any existing work.}
% \qingyun{Mention that this BAIc setting for MF-MAB has not been studied in existing work here or at the beginning of this paragraph. I added one sentence here. Please check.}

In Section~\ref{sec:regret-minimization}, we next study the regret minimization objective in \MFMAB.
% We consider a new regret definition in which the obtained reward depends on the pulled arm while the selected fidelities only influence how accurate the observed rewards are.
We introduce a novel definition of regret, wherein the rewards obtained are dependent on the pulled arms, while the selected fidelities solely affect the accuracy of the observed rewards.
The new regret definition covers applications like product management and advertisement distribution~\cite{han2020contextual,farris2015marketing} where the distributed advertisement (ad) determines the actual (but unknown) reward while the cost (fidelity) spent by the company on evaluating the impact of this ad decides the accuracy of the observed rewards (see Remark~\ref{rmk:application_regret_minimization}).
The difference of our regret definition and the one studied by~\citet{kandasamy2016multi} is discussed in Remark~\ref{rmk:new-regret-motivation}.
We first propose both a problem-independent (a.k.a., worst-case) regret lower bound \(\Omega(K^{1/3}\Lambda^{2/3})\) and a problem-dependent regret lower bound \(\Omega(K\log \Lambda)\) where the \(\Lambda\, (> 0)\) is the decision budget.
We then devise an elimination-based algorithm, which explores (and eliminates) arms in the highest fidelity \(M\) and exploits the remaining arms in the lowest fidelity \(1\).
The algorithm enjoys a \(\tilde{O}(K^{1/3}\Lambda^{2/3})\) problem-independent regret upper bound, which matches the corresponding lower bound up to some logarithmic factor,
and also a problem-dependent bound matching its corresponding lower bound tightly in a class of \MFMAB instances.

% !TeX root = ..\mf-mab.tex
\subsection{Related Works}
\label{sec:related-works}

Multi-fidelity multi-armed bandits (\MFMAB) was first proposed by~\citet{kandasamy2016multi}. They studied a cumulative regret minimization setting whose regret definition is different from ours (see Remark~\ref{rmk:new-regret-motivation} for more details).
Later, \citet{kandasamy2016gaussian} extended \MFMAB to bandits optimization, i.e., in continuous action space, with the objective of minimizing simple regret,
and \citet{kandasamy2017multi} further extended \citet{kandasamy2016gaussian} to the continuous fidelity space with the same objective of minimizing simple regret.
We note that the simple regret minimization in bandits optimization corresponds to the best arm identification with fixed budget in multi-armed bandits, and, therefore, is different from the cumulative regret minimization and best arm identification with fixed confidence objectives studied in this paper.

The multi-fidelity optimization was first introduced in simulating expensive environments via cheap surrogate models~\cite{huang2006sequential,forrester2007multi},
and, later on, was extended to many real-world applications, e.g., shape design optimization for ship~\cite{bonfiglio2018multi} or wing~\cite{zheng2013multi},
and wind farm optimization~\cite{rethore2014topfarm}.
Recently, the multi-fidelity idea was employed in hyperparameter optimization (HPO) for automated machine learning (autoML)~\cite{jamieson2016non,li2017hyperband,falkner2018bohb,li2020system}. The fidelity may correspond to training time, data set subsampling, and feature subsampling, etc.
These multi-fidelity settings help agents to discard some hyperparameter configurations with low cost (a.k.a., early discarding~\citep[\S3.1.3]{mohr2022learning}).
\citet{jamieson2016non} modeled hyperparameter optimization as non-stochastic best arm identification and applied the successive halving algorithm (SHA) to address it.
\citet{li2017hyperband} introduced Hyperband, a hyperparameter configuring and resources allocation method which employed SHA as its subroutine to solve real HPO tasks.
\citet{falkner2018bohb} proposed BOHB, which use Bayesian optimization method to select and update hyperparameters and employ SHA to allocate resources.
\citet{li2020system} extended SHA to asynchronous case for parallel hyperparameter optimization.
This line of works was based on the non-stochastic best arm identification problem~\cite{jamieson2016non} and, therefore, is very different from our stochastic modelling.

% !TeX root = ..\mf-mab.tex
\section{Model}\label{sec:model}

% \wei{It is unclear which are known and which are not known. I tried to added such explanation, please check.}
% \xuchuang{thanks, it is clear to me.}

We consider a \(K\,(\in\mathbb{N}^+)\)-armed bandit. Each arm can be pulled with \(M\,(\in\mathbb{N}^+)\) different fidelities. When an arm \(k\in \mathcal{K}\coloneqq\{1,\dots,K\}\) is pulled at fidelity \(m\in\mathcal{M}\coloneqq \{1,\dots,M\}\), one pays a cost \(\lambda\upbra{m}\,(>0)\) and observes a feedback value drawn from a \([0,1]\)-supported probability distribution with mean \(\mu_k\upbra{m}\). We assume the non-trivial case that \(\abs{\mu_k\upbra{m}-\mu_k\upbra{M}} \le \zeta\upbra{m}\), where \(\zeta\upbra{m} \ge 0\) is the observation error upper bound at fidelity \(m\).
Without loss of generality, we assume \(\lambda\upbra{1}\le \dots \le \lambda\upbra{M}\); otherwise, we can relabel the fidelity so that the \(\lambda\upbra{m}\) is non-decreasing with respect to \(m\).
We call the reward mean of an arm at the highest fidelity \(M\) as this arm's \emph{true reward mean}
and, without loss of generality, assume that these true reward means are in a descending order \(\mu_1\upbra{M} > \mu_2\upbra{M}\ge \mu_3\upbra{M}\ge \dots\ge \mu_K\upbra{M}\), and arm \(1\) is the unique optimal arm.

For the online learning problem, we assume that the reward distribution and the mean $\mu_k\upbra{m}$ of each arm $k$ at fidelity $m$ are unknown, while
the costs $\lambda\upbra{m}$'s and the error upper bounds $\zeta\upbra{m} $'s are known to the learning agent.
To summarize, a \MFMAB instance \(\mathcal{I}\) is parameterized by a set of tuples \((\mathcal{K}, \mathcal{M},(\lambda\upbra{m})_{m\in\mathcal{M}},(\zeta\upbra{m})_{m\in\mathcal{M}}, (\mu_k\upbra{m})_{(k,m)\in\mathcal{K}\times \mathcal{M}})\), with elements described above.

% \qingyun{Make it more explicit what consist of a \MFMAB instance as we used the term in Theorem 4.1. I added a simple description. Please check.}
% \wei{The above gap definition is only for the pure exploration case? Also, it feels like it is part of our solution, so it is better to put it in the next section?}
% \xuchuang{Yes, I updated it.}

In this paper, we consider two tasks on the above multi-fidelity bandit model: best arm identification and regret minimization. We will define these two tasks
in the respective sections below.

% !TeX root = ..\mf-mab.tex
\section{Best Arm Identification with Fixed Confidence}\label{sec:bai}

% \qingyun{Do we need to mention somewhere why we want to study BAI with Fixed Confidence instead of the other alternative setting of BAI, i.e., fixed budget?}
In this section, we study the best arm identification with fixed confidence (\BAIC) task in the multi-fidelity multi-armed bandits (\MFMAB) model.
% \qingyun{This is the first time `\MFMAB model' appears. It can be unclear to the others what is it even with the model section because it is not stated explicitly.}
% \xuchuang{Updated. thanks.}
The objective of \BAIC is to minimize the total budget spent for identifying the best arm with a confidence of at least \(1-\delta\).
As the cost of pulling an arm in \MFMAB depends on the chosen fidelity, we use the total cost, instead of total pulling times (sample complexity), as our criterion, and refer to it as \emph{cost complexity}.
% \footnote{For simplicity, we omit the term ``sample'' in this paper.}
We first present a cost complexity lower bound in Section~\ref{subsec:sample-cost-lower-bound}, then propose a LUCB algorithmic framework with two alternative procedures in Section~\ref{subsec:lucb-ucb}, and analyze their cost complexity upper bounds in Section~\ref{subsec:sample-cost-upper-bound}.
Lastly, we introduce concrete applications of the \BAIC problem in Section~\ref{subsec:application_bai}.

\subsection{Cost Complexity Lower Bound}\label{subsec:sample-cost-lower-bound}

We present the cost complexity lower bound in Theorem~\ref{thm:sample-cost-lower-bound}.
Its proof is deferred to Appendix~\ref{supapp:sample-cost-lower-bound}.

\begin{theorem}[Cost complexity lower bound]\label{thm:sample-cost-lower-bound}
  For any algorithm addressing the best arm identification with fixed confidence \(1-\delta\) for any parameter \(\delta>0\),
  any number of arms \(K\), any number of fidelities \(M\) with any observation error upper bound sequence
  % \qingyun{more precisely
  %     "accuracy sequence" here should be, "observation error upper bound sequence"?} 
  %     \xuchuang{Yes, you are right.}
  \((\zeta\upbra{1},\zeta\upbra{2},\dots,\zeta\upbra{M})\,(\zeta\upbra{M} = 0)\)
  and any cost sequence \((\lambda\upbra{1},\dots,\lambda\upbra{M})\),
  and any \(K\) fidelity subsets \(\{\mathcal{M}_1, \mathcal{M}_2,\dots, \mathcal{M}_K\}\) where the \(\mathcal{M}_k\) is a subset of full fidelity set \(\mathcal{M}\) containing the highest fidelity \(M\), i.e.,
  \(M\in \mathcal{M}_k\in 2^\mathcal{M},\) for all \(k\in\mathcal{K}\),
  there exists a \emph{\MFMAB} instance such that
  \[
    \begin{split}
      \E[\Lambda] \ge \left( \min_{m\in \mathcal{M}_1} \frac{\lambda\upbra{m}}{\KL(\nu_{1}\upbra{m}, \nu_{2}\upbra{M} + \zeta\upbra{m})}  + \sum_{k\neq 1}\min_{m\in \mathcal{M}_k}  \frac{\lambda\upbra{m}}{\KL(\nu_k\upbra{m}, \nu_1\upbra{M}-\zeta\upbra{m})} \right) \log \frac{1}{2.4\delta},
    \end{split}
  \]
  where \(\KL\) is the KL-divergence between two probability distributions,
  \(\nu_k\upbra{m}\) is the probability distribution associated with arm \(k\) when pulled at fidelity \(m\), and \(\nu \pm \zeta\) means to positively/negatively shift the distribution \(\nu\) by an offset \(\zeta\).
\end{theorem}

According to the KL-divergences' two inputs in the above lower bound, we define the reward gap \(\Delta_k\upbra{m}\) of arm \(k\) at fidelity \(m\) as follows,
\begin{equation}
  \label{eq:reward-gap}
  \Delta_k\upbra{m}\coloneqq \begin{cases}
    \mu_1\upbra{M} - (\mu_k\upbra{m} + \zeta\upbra{m} ) & \forall k \neq 1 \\
    (\mu_1\upbra{m}  - \zeta\upbra{m}) - \mu_2\upbra{M} & \forall k=1
  \end{cases}.
\end{equation}
For any suboptimal arm \(k\neq 1\), the gap \(\Delta_k\upbra{m}\) quantifies the distance between the optimal arm's reward mean \(\mu_1\upbra{M}\) and the suboptimal arm's reward mean upper bound at fidelity \(m\), i.e., \(\mu_k\upbra{m} + \zeta\upbra{m}\); and for the optimal arm \(1\), the gap \(\Delta_1\upbra{m}\) represents the distance between the optimal arm's reward mean lower bound at fidelity \(m\), i.e., \(\mu_1\upbra{m}-\zeta\upbra{m}\), and the second-best arm's reward mean \(\mu_2\upbra{M}\).

\begin{remark}
  If all reward distributions are assumed to be Bernoulli or Gaussian and let \(\mathcal{M}_k=\{m:\Delta_k\upbra{m}>0\}\), the regret bound can be simplified as
  \begin{equation}\label{eq:sample-cost-lower-bound}
    \E[\Lambda]  \ge C\sum_{k\in\mathcal{K}}\min_{m:\Delta_k\upbra{m} > 0}\frac{\lambda\upbra{m}}{(\Delta_k\upbra{m})^2} \log \frac{1}{\delta},
  \end{equation}
  where \(C>0\) is a constant depending on the reward distributions assumed.
  Especially, we denote \(m_k^* \coloneqq \argmin_{m:\Delta_k\upbra{m}>0} \frac{\lambda\upbra{m}}{(\Delta_k\upbra{m})^2}\), and with some algebraic transformations, it can be expressed as \begin{equation}
    \label{eq:efficient-fidelity}
    m_k^* = \argmax_{m\in\mathcal{M}} \frac{\Delta_k\upbra{m}}{\sqrt{\lambda\upbra{m}}}.
  \end{equation}
  With \(m_k^*\), the lower bound in~\eqref{eq:sample-cost-lower-bound} can be rewritten as \( \E[\Lambda]  \ge C\sum_{k\in\mathcal{K}} ({\lambda\upbra{m_k^*}}/{(\Delta_k\upbra{m_k^*})^2}) \log ({1}/{\delta})\), and we define the coefficient as
  \(H\coloneqq \sum_{k\in\mathcal{K}} {\lambda\upbra{m_k^*}}/{(\Delta_k\upbra{m_k^*})^2}\).
  The \(m_k^*\) can be interpreted as the \emph{optimal (most efficient) fidelity} for exploring arm \(k\).
  We note that the current lower bound in~\eqref{eq:sample-cost-lower-bound} does not contain the cost of finding this \(m_k^*\). This cost can be observed in our algorithm's cost complexity upper bound stated in Section~\ref{subsec:sample-cost-upper-bound}.
\end{remark}

\subsection{Algorithm Design}\label{subsec:lucb-ucb}
% \qingyun{Revise subsection title. E.g., An LUCB framework and cost complexity upper bounds}
In this subsection, we propose a Lower-Upper Confidence Bound (LUCB) algorithmic framework with two alternative procedures which employ different mechanisms to select suitable fidelities for arm exploration.
Generalized from the original (single-fidelity) LUCB algorithm~\citep{kalyanakrishnan2012pac},
the LUCB algorithmic framework in~\S\ref{subsubsec:lucb-framework} determines two critical arms (the empirical optimal arm \(\ell_t\) and second-best arm \(u_t\)) for exploration in each time slot \(t\).
However, in \MFMAB, with the critical arms suggested by the LUCB framework, one still needs to decide the fidelities for exploring the critical arms.
This fidelity selection faces an accuracy-cost trade-off, i.e., higher fidelity (accuracy) but suffering higher cost, or lower fidelity (accuracy) but enjoying lower cost.
This trade-off is different from the common exploration-exploitation trade-off in classic bandits.
Because the accuracy and costs are two orthogonal metrics, while in  exploration-exploitation trade-off, there is only a single regret metric.
We address the accuracy-cost trade-off in \S\ref{subsubsec:three-subroutine} with two alternative procedures.
% \xuchuang{Highlight the main challenge from high-level.}

% The novel techniques of our algorithm design are in the three exploration subroutines where one needs to decide

\subsubsection{LUCB Algorithmic Framework}\label{subsubsec:lucb-framework}

The main idea of LUCB~\citep{kalyanakrishnan2012pac} is to repeatedly select and explore two critical arms, that is, the empirical optimal arm \(\ell_t\) and the empirical second-best arm \(u_t\). When both critical arms' confidence intervals are separated---\(\ell_t\)'s LCB (lower confidence bound) is greater than \(u_t\)'s UCB (upper confidence bound), the algorithm terminates and outputs the estimated optimal arm \(\ell_t\).

The LUCB framework depends on a set of meaningful confidence intervals of the arms' rewards.
Usually, the confidence interval of an empirical mean estimate \(\hat{\mu}_k\upbra{m}\)
% in bandit algorithms \qingyun{in a stochastic bandit algorithm (or problem)?} 
can be expressed as \((\hat{\mu}_k\upbra{m}-\beta(N_{k,t}\upbra{m},t,\delta), \hat{\mu}_k\upbra{m} + \beta(N_{k,t}\upbra{m}, t, \delta)),\)
where \(\beta(N_{k,t}\upbra{m}, t, \delta)\) is the confidence radius and \(N_{k,t}\upbra{m}\) is the number of times of pulling arm \(k\) at fidelity \(m\) up to time \(t\) (include \(t\)).
As \MFMAB assumes that \(\abs{\mu_k\upbra{m} - \mu_k\upbra{M}} \le \zeta\upbra{m}\), based on observations of fidelity \(m\), the upper and lower confidence bounds for arm \(k\)'s true reward mean at the highest fidelity \(\mu_k\upbra{M}\) can be expressed as follows,
\begin{equation}\label{eq:lucb-at-m}
  \begin{split}
    \texttt{UCB}_{k,t}^{(m)} \coloneqq \hat{\mu}_{k,t}^{(m)} + \zeta^{(m)} + \beta(N_{k,t}\upbra{m}, t, \delta),\quad
    \texttt{LCB}_{k,t}^{(m)} \coloneqq \hat{\mu}_{k,t}^{(m)} - \zeta^{(m)} - \beta(N_{k,t}\upbra{m}, t, \delta),
  \end{split}
\end{equation}
where we set the confidence radius \(\beta(n,t,\delta) = \sqrt{{\log(Lt^4/\delta)}/{n}}\) and \(L\,(>0)\) is a factor.
Since the multi-fidelity feedback allows one to estimate an arm's true reward mean with observations of every fidelity, we pick the tightest one as arm \(k\)'s final confidence bounds as follows,
\begin{equation}\label{eq:lucb-all-m}
  \begin{split}
    \texttt{UCB}_{k,t} = \min_{m\in\mathcal{M}} \texttt{UCB}_{k,t}^{(m)},\quad
    \texttt{LCB}_{k,t} = \max_{m\in\mathcal{M}} \texttt{LCB}_{k,t}^{(m)}.
  \end{split}
\end{equation}

We use the above \(\texttt{UCB}_{k,t}\) and \(\texttt{LCB}_{k,t}\) formulas to select the two critical arms in each round and decide when to terminate the LUCB (Line~\ref{line:while-condition}).
We present the LUCB framework in Algorithm~\ref{alg:lucb-framework}.
The next step is to decide fidelities for exploring both critical arms in each round ( Line~\ref{line:explore}).

\begin{algorithm}[tb]
  \caption{LUCB Framework for Multi-Fidelity \BAIC}\label{alg:lucb-framework}
  \begin{algorithmic}[1]
    \State \textbf{Input:} violation probability \(\delta\)
    \State \textbf{Initialization: } \(\hat{\mu}_{k,t}\upbra{m}=0\), \(N_{k,t}\upbra{m} = 0, \texttt{UCB}_{k,t}=1,\texttt{LCB}_{k,t}=0\) for all arm \(k\) and fidelity \(m\), and \(\ell_t=1, u_t=2\), \(t=1,\) $\tilde{\mu}_1\upbra{M}, \tilde{\mu}_2\upbra{M}.$

    \While{\(\texttt{LCB}_{\ell_t,t} \le \texttt{UCB}_{u_t,t}\)}\label{line:while-condition}

    \State \(\displaystyle \ell_t  \gets \argmax_{k \in \mathcal{K}} \texttt{UCB}_{k,t},\, u_t  \gets \argmax_{k \in \mathcal{K}\setminus \{\ell_t\}} \texttt{UCB}_{k,t}
    \) \Comment{Select top two \texttt{UCB} indexes}

    \State \textsc{Explore}(\(u_t\)) and \textsc{Explore}(\(\ell_t\))
    \Comment{Any procedure of Algorithm~\ref{alg:fidelity-selection-procedure}}
    \label{line:explore}

    % \State \(t\gets t+2\) \Comment{}
    % \COMMENT{Find the second largest UCB}
    % \State \(m_{\ell_t, t}, m_{u_t,t}\gets \text{any procedure of Algorithm~\ref{alg:fidelity-selection-procedure}}\)\label{line:fidelity-selection}
    \EndWhile
    \State \textbf{Output:} arm \(\ell_t\)
  \end{algorithmic}
\end{algorithm}

\subsubsection{Exploration Procedures}\label{subsubsec:three-subroutine}

% We illustrate three alternative procedures in this section to select fidelities for exploration and present them in Algorithm~\ref{alg:fidelity-selection-procedure}.

To address the accuracy-cost trade-off in fidelity selections,
we devise a UCB-type policy which ``finds'' the optimal fidelity \(m_k^*\) in~\eqref{eq:efficient-fidelity} for each arm \(k\) (\explore{A})
and an explore-then-commit policy stopping at a good fidelity that is at least half as good as the optimal one (\explore{B}).

% , that is, it spends most cost of arm \(k\)'s exploration on the optimal fidelity \(m_k^*\) of arm \(k\)'s exploration.
% We note that constructing the UCB index involves the estimation of \(\Delta_k\upbra{m}\), i.e., the differences of two arms' reward means, which is different from the classic UCB's estimating one arm's reward mean, and, thus, a novel challenge.
% Thirdly, we propose an explore-then-commit policy.
% Instead of finding the optimal fidelity \(m_k^*\) (which may need a high cost), we devise an exploration-stop condition such that the policy stops (commits) at a good fidelity for each arm that is at least half as good as the optimal one.

\textbf{Notations.} The lower bound in~\eqref{eq:sample-cost-lower-bound} implies that there exists an optimal fidelity \(m_k^*\) for exploring arm \(k\).
As \eqref{eq:efficient-fidelity} shows, the optimal fidelity \(m_k^*\) maximizes the \({\Delta_k\upbra{m}}/{\sqrt{\lambda\upbra{m}}}\),
where the \(\Delta_k\upbra{m}\) defined in~\eqref{eq:reward-gap} consists of two unknown reward means: \(\mu_k\upbra{m}\) and \(\mu_1\upbra{M}\) (or \(\mu_2\upbra{M}\)).
That is, calculating \(\Delta_k\upbra{m}\) for all \(k\) needs the top two arms' reward means \(\mu_1\upbra{M}\) and \(\mu_2\upbra{M}\) which are unknown a priori.
To address the issue, we assume the knowledge of an upper bound of the optimal arm's reward mean \(\tilde{\mu}_1\upbra{M}\)
and a lower bound of the second-best arm's reward mean \(\tilde{\mu}_2\upbra{M}\) (see Remark~\ref{rmk:input-mu-1-2} for how to obtain \(\tilde{\mu}_1\upbra{M}\) and \(\tilde{\mu}_2\upbra{M}\) in real-world applications).
With the \(\mu_1\upbra{M}\) and \(\mu_2\upbra{M}\) replaced by \(\tilde{\mu}_1\upbra{M}\) and \(\tilde{\mu}_2\upbra{M}\), we define the ancillary reward gaps \(\tilde{\Delta}_k\upbra{m}\) and the {ancillary optimal fidelity} \(\tilde{m}_k^*\)  as follows,
\begin{equation*}\label{eq:tilde-Delta}
  \tilde{\Delta}_{k}\upbra{m}\coloneqq
  \begin{cases}
    \tilde{\mu}_1\upbra{M} - ({\mu}_{k}\upbra{m} + \zeta\upbra{m}) & \forall k \neq 1 \\
    ({\mu}_{k}\upbra{m} - \zeta\upbra{m}) - \tilde{\mu}_2\upbra{M} & \forall k= 1
  \end{cases},\quad\text{ and }\quad
  \tilde{m}_k^* \coloneqq \argmax_{m\in\mathcal{M}} \frac{\tilde{\Delta}_k\upbra{m}}{\sqrt{\lambda\upbra{m}}}.
\end{equation*}
% \wei{What are $\tilde{\mu}_1$ and $\tilde{\mu}_2$ above? Are they missing the superscript?}
Especially, we have \(m_k^* \ge \tilde{m}_k^*\) because the cost \(\lambda\upbra{m}\) is non-decreasing and the replacement enlarges the numerator of the \(\argmax\) item in~\eqref{eq:efficient-fidelity}, and,
when the bounds \(\tilde{\mu}_1\upbra{M}\) and \(\tilde{\mu}_2\upbra{M}\) are close the the true reward means, we have \(m_k^* = \tilde{m}_k^*\);
hence, as \(\Delta_k\upbra{m_k^*} > 0\), we assume \(\Delta_k\upbra{\tilde{m}_k^*} > 0\) for all arms \(k\) as well.
To present the next two procedures,
we define the estimate of \(\tilde{\Delta}_k\upbra{m}\) as follows,
\begin{equation*}\label{eq:hat-Delta}
  \hat{\Delta}_{k,t}\upbra{m}\coloneqq
  \begin{cases}
    \tilde{\mu}_1\upbra{M} - (\hat{\mu}_{k,t}\upbra{m} + \zeta\upbra{m}) & \forall k \neq \ell_t \\
    (\hat{\mu}_{k,t}\upbra{m} - \zeta\upbra{m}) - \tilde{\mu}_2\upbra{M} & \forall k= \ell_t
  \end{cases},
\end{equation*}
where the \(\ell_t\) is the estimated optimal arm by LUCB.
For simplify, we omit the input \(\ell_t\) for \(\hat{\Delta}_{k,t}\upbra{m}(\ell_t)\) in the LHS of this definition and thereafter.

% Given the bounds of top two arms' reward means \(\tilde{\mu}_1\upbra{M}\) and \(\tilde{\mu}_2\upbra{M}\), we
% where \(\tilde{\Delta}_k\upbra{m}\) is defined as \(\tilde{\mu}_1\upbra{M} - ({\mu}_{k}\upbra{m} + \zeta\upbra{m})\) for arm  \(k \neq 1\), and \(({\mu}_{k}\upbra{m} - \zeta\upbra{m}) - \tilde{\mu}_2\upbra{M}\) otherwise.

\paragraph{\explore{A}}
We devise the \({\Delta_k\upbra{m}}/{\sqrt{\lambda\upbra{m}}}\)'s upper confidence bounds (\texttt{f-UCB}) as follows,
for any fidelity \(m\in\mathcal{M}\) and arms \(k\in\mathcal{K}\),
\(
% \label{eq:fidelity-ucb}
\texttt{f-UCB}_{k,t}\upbra{m} \coloneqq
{\hat{\Delta}_{k,t}\upbra{m}}/{\sqrt{\lambda\upbra{m}}} + \sqrt{{2\log N_{k,t}}/({\lambda\upbra{m}N_{k,t}\upbra{m}})},
\)
where the \(N_{k,t} \coloneqq \sum_{m\in\mathcal{M}}N_{k,t}\upbra{m}\) is the total number of times of pulling arm \(k\) up to time \(t\).
Whenever the arm \(k\) is selected by LUCB, we pick the fidelity \(m\) that maximizes its \(\texttt{f-UCB}_{k,t}\upbra{m}\) to explore it (see Line~\ref{line:argmax-f-ucb} in Algorithm~\ref{alg:fidelity-selection-procedure}).
For any arm \(k\), this policy guarantees that most of the arm's pulling are on its estimated optimal fidelity \(\tilde{m}_k^*\), or formally, \(N_{k,t}\upbra{m} = O(\log ( N_{k,t}\upbra{\tilde{m}_k^*}))\) for any fidelity \(m\neq \tilde{m}_k^*\) as Lemma~\ref{lma:f-ucb-property} in Appendix shows.
Therefore, \explore{A} spends most cost on the fidelity \(\tilde{m}_k^*\).
% \todo{think about better naming}

\paragraph{\explore{B}} The cost of finding the estimated optimal fidelity \(\tilde{m}_k^*\) in \explore{A} can be large.
To avoid this cost, we devise another approach that stops exploration when finding a \emph{good} fidelity \(\hat{m}_k^*\).
That is, instead of finding the \(\tilde{m}_k^*\) that maximizes \(\tilde{\Delta}_k\upbra{m}/\sqrt{\lambda\upbra{m}}\), we stop at a fidelity \(\hat{m}_k^*\) whose \(\tilde{\Delta}_k\upbra{m}/\sqrt{\lambda\upbra{m}}\) is at least half as large as that of \(\tilde{m}_k^*\), i.e., \(
{\tilde{\Delta}_k\upbra{\hat{m}_k^*}}/{\sqrt{\lambda\upbra{\hat{m}_k^*}}} \ge ({1}/{2}) ({\tilde{\Delta}_k\upbra{\tilde{m}_k^*}}/{\sqrt{\lambda\upbra{\tilde{m}_k^*}}}).
\)
We prove that the above inequality holds for \(\hat{m}_{k}^* = \argmax_{m\in\mathcal{M}} {\hat{\Delta}_{k,t}\upbra{m}}/{\sqrt{\lambda\upbra{m}}}\) when the condition in Line~\ref{line:fidelity-commit-condition} of Algorithm~\ref{alg:fidelity-selection-procedure}
% , i.e., \( {\hat{\Delta}_{k,t}\upbra{\hat{m}_k^*}}/{\sqrt{\lambda\upbra{\hat{m}_k^*}}} > 3\sqrt{{\log(L/\delta)}/{\lambda\upbra{1}N_{k,t}\upbra{m}}},\) 
holds (see Lemma~\ref{lma:good-fidelity-guarantee}).
Hence, for each arm \(k\), \explore{B} explores it at all fidelities uniformly, and, when the condition in Line~\ref{line:fidelity-commit-condition} holds, \explore{B} finds a good fidelity \(\hat{m}_k^*\) and keeps choosing \(\hat{m}_k^*\) for exploring arm \(k\) since then.

\begin{algorithm}[tb]
  \caption{\textsc{Explore} Procedures}
  \label{alg:fidelity-selection-procedure}
  \begin{algorithmic}[1]

    % \Procedure{Explore-A}{$k$}
    % % \State \textbf{Input:} $\zeta\upbra{1},\zeta\upbra{2},\dots,\zeta\upbra{M}$
    % \State \(m_{k,t}\gets \min\left\{ m \left| \beta(N_{k,t}\upbra{m},t,\delta)\ge \zeta^{(m)}\right.\right\}\)
    % \State Pull \(\displaystyle \left(k, m_{k,t}\right)\), observe reward, and update corresponding statistics
    % \EndProcedure

    \Procedure{\explore{A}}{$k$}
    % \State \textbf{Input:} $\tilde{\mu}_1\upbra{M}, \tilde{\mu}_2\upbra{M}$

    \State \(m_{k, t}\gets
    \argmax_{m\in\mathcal{M}}
    \texttt{f-UCB}_{k,t}\upbra{m}
    \) \label{line:argmax-f-ucb}
    % \text{ via~\eqref{eq:fidelity-ucb}}
    \State Pull \(\displaystyle \left(k, m_{k,t}\right)\), observe reward, and update corresponding statistics
    \EndProcedure

    \Procedure{\explore{B}}{$k$}
    % \State \textbf{Input:} $\tilde{\mu}_1\upbra{M}, \tilde{\mu}_2\upbra{M}$

    \If{\(\texttt{isFixed}_{k}\)}

    \State Pull \(\displaystyle \left(k, \hat{m}_{k}^*\right)\), observe reward, and update corresponding statistics

    \Else

    \For{each fidelity \(m\)}
    \State Pull \((k, m)\), observe reward, and update corresponding statistics
    \EndFor

    \If {\(\max_{m\in\mathcal{M}} \frac{\hat{\Delta}_{k,t}\upbra{m}}{\sqrt{\lambda\upbra{m}}} > 3\sqrt{\frac{\log(L/\delta)}{\lambda\upbra{1}N_{k,t}\upbra{m}}}\)} \label{line:fidelity-commit-condition}
    \State \(\hat{m}_{k}^* \gets \argmax_{m\in\mathcal{M}} \frac{\hat{\Delta}_{k,t}\upbra{m}}{\sqrt{\lambda\upbra{m}}}\) and \(\texttt{isFixed}_{k}\gets \texttt{True}\)
    \EndIf

    \EndIf

    \EndProcedure

  \end{algorithmic}
\end{algorithm}

\subsection{Cost Complexity Upper Bound Analysis}\label{subsec:sample-cost-upper-bound}

In the following, we present the cost complexity upper bounds of Algorithm~\ref{alg:lucb-framework} with above two procedures in
Theorem~\ref{thm:sample-cost-upper-bound} respectively.
We first present a technical assumption as follows.

\begin{assumption}\label{asp:arm-2-is-good}
  The reward mean of the second-best arm \(2\) is no less than the reward mean upper bound of any suboptimal arm \(k\) in fidelity \(m_k^*\). It can be expressed as follows,
  \begin{equation}\label{eq:arm-2-is-good}
    \mu_2\upbra{M} \ge \mu_k\upbra{m_k^*} + \zeta\upbra{m_k^*},\quad\forall k\neq 1.
  \end{equation}
\end{assumption}

Note that the definition of \(m_k^*\) in \eqref{eq:efficient-fidelity} implies that \(\mu_1\upbra{M} > \mu_k\upbra{m_k^*} + \zeta\upbra{m_k^*},\forall k \neq 1\).
Comparing this to \eqref{eq:arm-2-is-good},
Assumption~\ref{asp:arm-2-is-good} means that
the true reward mean of the second-best arm \(2\) is close to that of the optimal arm \(1\).
Next, we present the cost complexity upper bounds in Theorem~\ref{thm:sample-cost-upper-bound}.

\begin{theorem}[Cost complexity upper bounds for Algorithm~\ref{alg:lucb-framework} with procedure in Algorithm~\ref{alg:fidelity-selection-procedure}]
  \label{thm:sample-cost-upper-bound}
  Given Assumption~\ref{asp:arm-2-is-good} and \(L \ge 4KM\), Algorithm~\ref{alg:lucb-framework} outputs the optimal arm with a probability at least \(1-\delta\).
  The cost complexities of Algorithm~\ref{alg:lucb-framework} with different fidelity selection procedures in Algorithm~\ref{alg:fidelity-selection-procedure} are upper bounded as follows,
  \begin{alignat}{2}
    \text{\emph{(\explore{A})}}\qquad &  & \E[\Lambda] & = O\left( \tilde{H}\log \left( \frac{L(\tilde{H}+\tilde{G})}{\lambda\upbra{1}\delta} \right) + \tilde{G} \log \log \left(\frac{L(\tilde{H}+\tilde{G})}{\lambda\upbra{1}\delta}\right) \right),\label{eq:sample-cost-upper-bound-A} \\
    \text{\emph{(\explore{B})}}\qquad &  & \E[\Lambda] & = O\left(\tilde{H} \sum_{m\in\mathcal{M}} \frac{\lambda
        \upbra{m}}{\lambda\upbra{1}} \log\left(  \frac{\sum_{m\in\mathcal{M}} (\lambda\upbra{m}/\lambda\upbra{1}) \tilde{H} L}{\lambda\upbra{1}\delta} \right) \right),\label{eq:sample-cost-upper-bound-B}
  \end{alignat}
  where \(\tilde{H} \coloneqq \sum_{k\in\mathcal{K}}\frac{\lambda\upbra{\tilde{m}_k^*}}{(\tilde{\Delta}_k\upbra{\tilde{m}_k^*})^2},\)
  and \(
  \tilde{G} \coloneqq \sum_{k\in \mathcal{K}} \sum_{m\neq \tilde{m}_k^*}  \left( \frac{\tilde{\Delta}_k\upbra{\tilde{m}_k^*}}{\sqrt{\lambda\upbra{\tilde{m}_k^*}}} - \frac{\tilde{\Delta}_k\upbra{m}}{\sqrt{\lambda\upbra{m}}} \right)^{-2}.
  \)

\end{theorem}

\begin{wrapfigure}{r}{0.6\textwidth}
  \centering
  \vspace{-10pt}
  \subfloat[\explore{A} is better\label{fig:explore-A-vs-B-1}]{\includegraphics[width=0.3\textwidth]{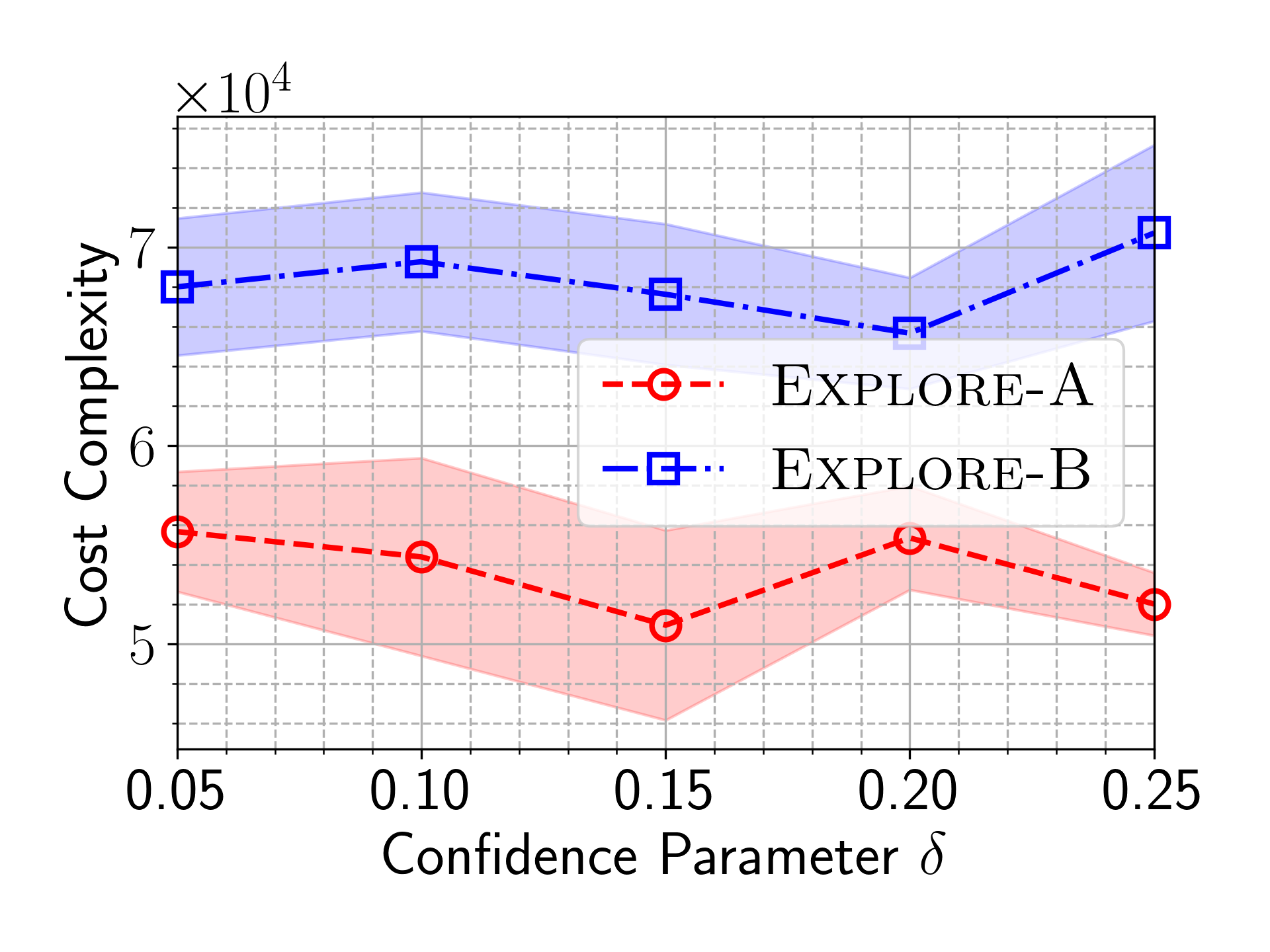}}
  \subfloat[\explore{B} is better\label{fig:explore-A-vs-B-2}]{\includegraphics[width=0.3\textwidth]{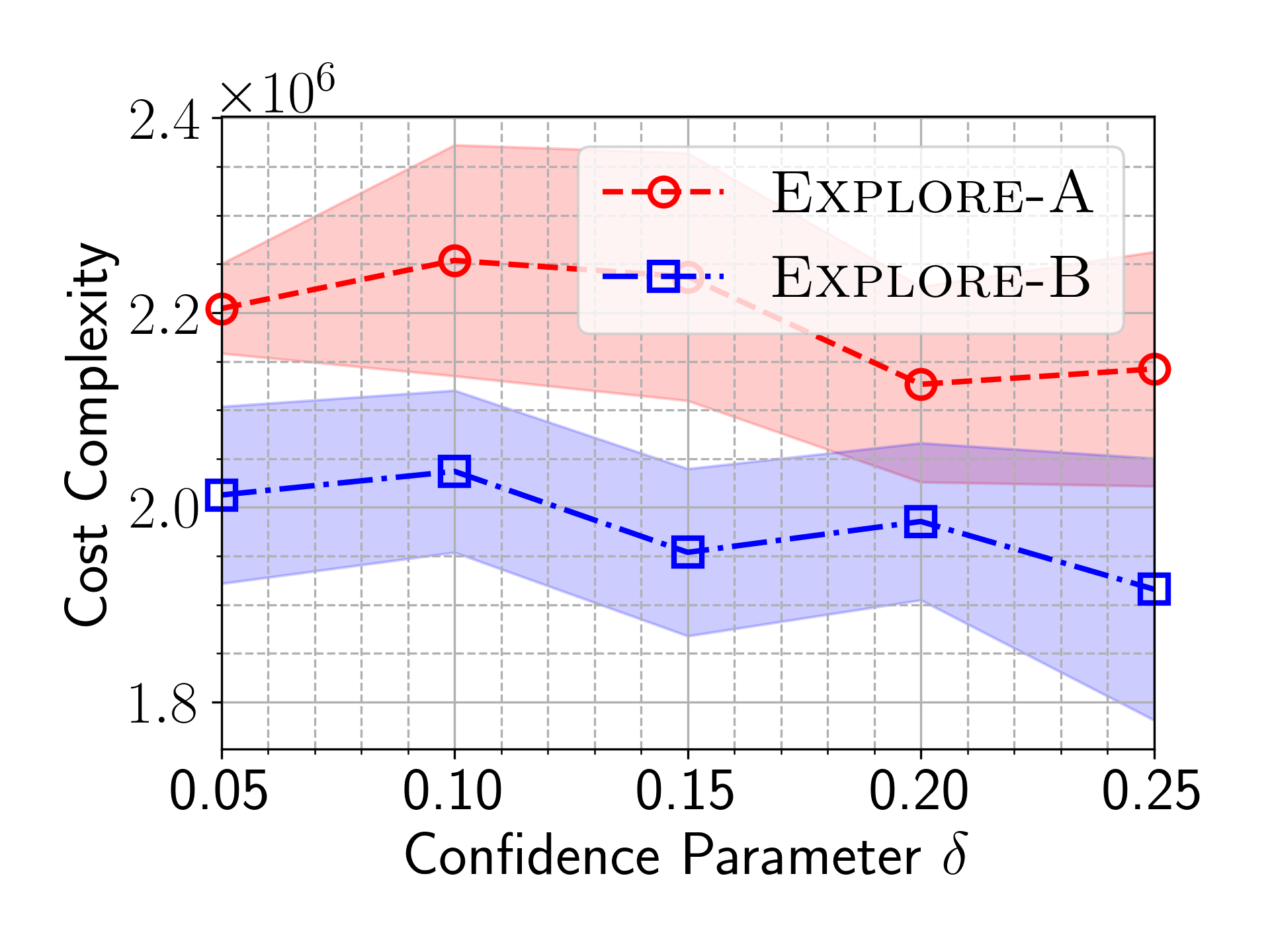}}
  \vspace{-3pt}
  \caption{\explore{A} vs. \explore{B}}\label{fig:explore-A-vs-B}
  \vspace{-15pt}
\end{wrapfigure}

\textbf{\explore{A} \textit{vs.} \explore{B}}
When \(\tilde{G} = O(M\tilde{H})\), the cost complexity upper bound of \explore{A} is less than that of \explore{B} (see Figure~\ref{fig:explore-A-vs-B-1}).
However, when \(\tilde{G}\) is far more larger than \(M\tilde{H}\), \explore{B} is better (see Figure~\ref{fig:explore-A-vs-B-2}). For example, when there is a fidelity \(m'\,(\neq \tilde{m}_k^*)\)
whose \(\Delta_k\upbra{m'}/\sqrt{\lambda\upbra{m'}}\) is very close to that of fidelity \(\tilde{m}_k^*\)
(the case in Figure~\ref{fig:explore-A-vs-B-2}),
this \(\tilde{G}\) would be very large
because \explore{A} needs to pay a high cost to distinguish fidelity \(m'\) from \(\tilde{m}_k^*\);
while in this scenario, \explore{B} stops by either \(m_k'\) or \(\tilde{m}_k^*\) since their \(\tilde{\Delta}_k\upbra{m}/\sqrt{\lambda\upbra{m}}\) are similar and, therefore, enjoys a smaller cost complexity upper bound.
We report the numerical comparisons between both procedures in Figure~\ref{fig:explore-A-vs-B}.
The detailed setup of the simulations is given in Appendix~\ref{sec:simulation-detail}.
% \wei{To save space, you can put the two figures to the right of this paragraph, in the wrap-around mode.}

\begin{remark}[Tightness of cost complexity bounds]
  The first term of cost complexity upper bound for \explore{A}
  % \wei{Is this \textsc{Explore}-A or \explore{A}?}
  in \eqref{eq:sample-cost-upper-bound-A} matches the cost complexity lower bound in \eqref{eq:sample-cost-lower-bound} up to a constant
  when \(\tilde{m}_k^* = m_k^*\) and \(\tilde{H} = H\)
  (i.e., when \(\tilde{\mu}_1\upbra{M}\) and \(\tilde{\mu}_2\upbra{M}\) are close to their ground truth values).
  The cost complexity upper bound of \explore{B} in~\eqref{eq:sample-cost-upper-bound-B} matches the lower bound with an additional \(\sum_{m\in\mathcal{M}} \lambda\upbra{m}/\lambda\upbra{1}\) coefficient when \(\tilde{m}_k^* = m_k^*\) and \(\tilde{H} = H\).
\end{remark}

\begin{remark}[Comparison to classic MAB's sample complexity]
  If we reduce our cost complexity upper bound result in \MFMAB to classic (single-fidelity) MAB, i.e., letting \(M=1, \lambda\upbra{m}=1\), then both cost complexity upper bounds reduce to \(O(\sum_{k}({1}/{\Delta_k^2})\log(1/\delta))\) where \(\Delta_k \coloneqq \mu_1 - \mu_k\), which is exactly the classic sample complexity upper bound for (single-fidelity) \BAIC~\cite{mannor2004sample,kaufmann2016complexity}.
\end{remark}

\subsection{Application}\label{subsec:application_bai}
One typical application of the best arm identification problem is hyperparameter optimization~\cite{hutter2019automated,elsken2019neural} (including neural architecture search) for machine learning, in which the goal is to identify the best hyperparameter configuration---the training set-up for a machine learning model attaining the best predictive performance---with as low resource as possible. A mapping between concepts in this application and \MFMAB is discussed as follows.
\textbf{Arm:} Hyperparameter configurations of machine learning models, e.g., neural network architectures.
\textbf{Reward:}
Predictive performance of the resulting machine learning model trained based on the selected configuration (arm).
\textbf{Fidelity dimension:} For a particular hyperparameter configuration (arm), one typically has the choice to determine a certain level of resources to allocate for training the model. The concept of  ``training resource'' can be considered the fidelity dimension. More concretely, commonly used training resources include \emph{the number of epochs} and \emph{the training data sample}, both of which satisfy our cost assumption. For example, the larger the number of epochs or training data samples, the more expensive to train the model.
% \begin{itemize}[leftmargin=*]
%   \vspace{-4mm}
%   \setlength\itemsep{-0.2em}
%   \item Arm: Hyperparameter configurations of machine learning models, e.g., neural network architectures.
%   \item Reward:
%         % Each time a hyperparameter configuration, e.g., a particular neural network architecture, is selected,
%         Predictive performance of the resulting machine learning model trained based on the selected configuration (arm).
%         %  could be considered as the reward.
%   \item Fidelity dimension: For a particular hyperparameter configuration (arm), one typically has the choice to determine a certain level of resources to allocate for training the model. The concept of  `training resource' can be considered the fidelity dimension. More concretely, commonly used training resources include \emph{the number of epochs} and \emph{the training data sample}, both of which satisfy our cost assumption. For example, the larger the number of epochs or training data samples, the more expensive to train the model.
% \end{itemize}

\begin{remark}[On the assumptions of multi-fidelity feedback in the hyperparameter optimization application]
  Observation error upper bound \(\zeta\upbra{m}\) is the maximum distance from resources allocated to the terminal validation loss. According to benchmarked results in two recent benchmarks for multi-fidelity hyperparameter optimization, including \texttt{HPOBench}~\cite{eggensperger2021hpobench} and \texttt{YAHPO Gym}~\cite{pfisterer2022yahpo}, under the typically used fidelity dimension, including the number of epochs, the training data sample, etc., the maximum distance from the terminal validation loss often decreases with the increase of resources, i.e., \(\zeta\upbra{m}\) decreases with the increase of \(m\). Thanks to these benchmarks, it is also convenient to know \(\zeta\upbra{m}\) under different fidelities \(m\) for commonly used types of fidelity dimension.
\end{remark}

\begin{remark}[On the bounds \(\tilde \mu_1\upbra{M}\) and \(\tilde \mu_2\upbra{M}\) utilized in Algorithm~\ref{alg:fidelity-selection-procedure} in the hyperparameter optimization application]\label{rmk:input-mu-1-2}
  Although the exact reward means are typically not accessible, it is easy to get a good approximation of them satisfying our requirements based on domain knowledge. For example, in an image classification task, an easy approximation of the best arm's reward means is to use the reward from a perfect classification, i.e., \(\tilde \mu_1\upbra{M} = 1.0\), which clearly satisfies \(\tilde \mu_{1}\upbra{M} \geq \mu_{1}\upbra{M}\). {For the reward of the second-best arm, we can use the performance of a commonly used model that has a fairly good performance based on benchmarked results as a good approximation.}
  % For the second-best arm's reward, we can use the state-of-the-art performance on this specific task. For example, state-of-the-art image classification accuracy on the ImageNet\footnote{\url{https://paperswithcode.com/sota/image-classification-on-imagenet}} is 0.91 and 
  One can easily find the benchmarked performance of commonly used models on a wide range of well-defined machine learning tasks on the \texttt{Papers with code} website.\footnote{\url{https://paperswithcode.com/sota}}
  % \wei{Why does the state-of-the-art performance corresponds to the second best arm? I do not see a direct connection.}
  For novel tasks without well-benchmarked results, one pragmatic way to get \(\tilde \mu_2\upbra{M}\) is to use the result from a particular default machine learning model without any tuning.
\end{remark}

\section{Regret Minimization}\label{sec:regret-minimization}

% \subsection{New Regret Definition} 
In this section, we study the regret minimization objective: given a budget \(\Lambda\in \mathbb{R}^+\), minimize the regret---the cumulative difference between the optimal policy's rewards and an algorithm's.
We define the reward obtained in each time slot as the pulled arm's true reward mean (realized at the highest fidelity, but \emph{unrevealed} to the learner), no matter at which fidelity the arm is pulled, while the learner's observation depends on the pulled fidelity as Section~\ref{sec:model} shows.
Under this reward definition, the optimal policy is to constantly pull the optimal arm \(1\) with the lowest fidelity \(m=1\).
Consequently, the expected regret can be expressed as follows,
\begin{equation}\label{eq:lambda1-regret}
    \mathbb{E}[R(\Lambda)]
    \coloneqq \frac{\Lambda}{\lambda^{(1)}} \mu_1\upbra{M}
    - \mathbb{E}\left[ \sum_{t=1}^N \mu_{I_t}^{(M)} \right],
\end{equation}
where \(N\coloneqq \max\{n: \sum_{t=1}^n \lambda\upbra{m_t} \le \Lambda\}\) is the total number of time slots, and \(I_t\) is the arm pulled by a concerned algorithm at time slot \(t\). Next, we illustrate the regret definition's real-world applications in Remark~\ref{rmk:application_regret_minimization}.

\begin{remark}[Applications of the new regret definition]\label{rmk:application_regret_minimization}
    One typical application of the regret minimization problem under \MFMAB is a variant of the advertisement distribution problem~\cite{han2020contextual, farris2015marketing}. In this problem, the objective is to maximize the total return from all the distributed ads within a fixed marketing budget (e.g., in terms of money). We have the following mapping between the application-specific concepts and concepts in \texttt{MF-MAB}.
    \textbf{Arm:} The ads to distribute are the arms.
    \textbf{Reward:} The return from each of the ads, once distributed, is the corresponding ground-truth reward \(\mu_k\upbra{M}\).
    \textbf{Low fidelity:} A minimum cost is needed every time any ad is distributed. For example, the minimum cost may include the necessary resource needed to ensure that the ad satisfies legal and regulatory requirements and to distribute the ad on the designated platform. This minimum cost can be considered the lowest fidelity cost, which can never be waived. \textbf{High fidelity:} since the expected return from different ads can be vastly different, one needs a good estimation on the expected return so as to select the profitable ads to distribute. The cost needed to get a reliable estimate of the expected return can be considered the highest fidelity cost. For example, doing a large-scale user study/survey, and/or consulting experts can give a good estimate of the expected return from the ads, which, however, is resource-consuming.
    % Hence, one should spend this resource wisely compared to the alternative of directly distributing the ads with only the lowest fidelity cost.

    This type of application is also common in production management with uncertainty where without knowing the expected return of the concerned products, the decision maker faces the two options of (1) spending the minimum resource needed to directly produce certain products; and (2) spending more resource to first get a good estimates of the expected returns from the different options and then put the ones with the highest expected returns into production.
\end{remark}

\begin{remark}[Comparison to regret definition of~\citet{kandasamy2016multi}]
    \label{rmk:new-regret-motivation}
    \citet{kandasamy2016multi} defined the per time slot reward as the pulled arms' true reward mean multiplied by the cost, i.e., \(\lambda\upbra{m_t}\mu_{I_t}\upbra{M}\), and defined their regret as, \(
    \E[R'(\Lambda)] \coloneqq \Lambda\mu_1\upbra{M} - \E[ \sum_{t=1}^N \lambda\upbra{m_t}\mu_{I_t}\upbra{M} ].
    \)
    We note that multiplying the reward mean with the fidelity-level cost, \(\lambda\upbra{m_t}\mu_{I_t}\upbra{M}\), does not fit into the applications in Remark~\ref{rmk:application_regret_minimization}, and thus we provide an
    alternative definition in \eqref{eq:lambda1-regret} to fit our needs.
    % \wei{Added a sentence here. Please check.}
    Comparing the formula of both regret definitions, we have \(\mathbb{E}[R'(\Lambda)] \le \lambda\upbra{1}\mathbb{E}[R(\Lambda)].\)
    Note that both regret definitions are very different, so as their bound analysis and algorithm design.
\end{remark}

\subsection{Theoretical Regret Bounds}

We present the problem-independent and -dependent regret upper and lower bounds on the new regret definition in Table~\ref{tab:regret-bounds}.
Due to the limit of space, we defer their detailed statements to Appendix~\ref{sec:regret-minimization-long}.
The upper bounds are achieved by a 2-phase algorithm.
In the first phase, we devise an elimination algorithm for \MFMAB based on~\citet{auer2010ucb} which always explore arms in the highest fidelity \(m= M\),
and, in the second phase, the algorithm exploits the remaining arms left from the first phase in turn in the lowest fidelity \(m = 1\).
We present the details of Algorithm~\ref{alg:elimination-for-rm} at Appendix~\ref{sec:regret-minimization-long}.

\begin{table}[t]
    \centering
    \caption{Bounds on the regret defined in~\eqref{eq:lambda1-regret}: the problem-dependent lower and upper bounds are for \(\liminf_{\Lambda\to\infty}\mathbb{E}[R(\Lambda)]/\log \Lambda\) and \(\limsup_{\Lambda\to\infty}\mathbb{E}[R(\Lambda)]/\log \Lambda\) respectively.}
    \label{tab:regret-bounds}
    \resizebox{\columnwidth}{!}{
        \begin{tabular}{lll}
            \toprule
             & Problem-independent (a.k.a., worst case)
             & Problem-dependent (\(\lim_{\Lambda\to\infty}\mathbb{E}[R(\Lambda)]/\log \Lambda\))
            \\ \midrule
            Lower bound
             & \(\displaystyle \Omega \left( K^{1/3}\Lambda^{{2/3}} \right)\)
            (Theorem~\ref{thm:regret-lower-bound})
             & \(\displaystyle \sum_{k\in \mathcal{K}} \min_{m:\Delta_k^{(m)}>0} \left( \frac{\lambda\upbra{m}}{\lambda\upbra{1}} \mu_1\upbra{M} - \mu_k\upbra{M} \right)\frac{C}{(\Delta_k^{(m)})^2}\) (Theorem~\ref{thm:problem-dependent-regret-lower-bound})
            \\
            Upper bound
             & \(\displaystyle O\left( K^{{1/3}} \Lambda^{{2/3}} (\log \Lambda)^{{1/3}} \right)\) (Theorem~\ref{thm:regret-upper-bound})
             & \(\displaystyle \sum_{k\in \mathcal{K}}
            \left( \frac{\lambda\upbra{M}}{\lambda\upbra{1}} \mu_1\upbra{M} - \mu_k\upbra{M} \right)
            \frac{16}{(\Delta_k\upbra{M})^2}\) (Theorem~\ref{thm:problem-dependent-regret-upper-bound})
            \\
            \bottomrule
        \end{tabular}
    }
\end{table}

\begin{remark}[Real-world implication of the 2-stage algorithm design]
    There are real-world applications, e.g., the advertisement distribution problem in Remark~\ref{rmk:application_regret_minimization}, where the explorations are conducted at the high fidelity (e.g., a large-scale user study) and the exploitations are conducted at the low fidelity (e.g., advertisement distribution via some platforms). This corroborates our algorithm design which also explores at high fidelity and exploits at low fidelity.
    On the other hand, the fact that our algorithm enjoys the tight regret performance comparing to regret lower bound (both problem-independent and -dependent, see Remarks~\ref{rmk:tightness-problem-independent-regret-bounds} and~\ref{rmk:tightness-problem-dependent-regret-bounds}) also implies that the approach of first conducting large-scale user study and then massive distributing good ads used in real-world advertisement distribution is reasonable.
\end{remark}

\begin{remark}[Tightness of problem-independent regret bounds]\label{rmk:tightness-problem-independent-regret-bounds}
    The problem-independent regret upper bound matches the problem-independent lower bound in terms of the number of arms \(K\) and up to some logarithmic factor in terms of the budget \(\Lambda\).
\end{remark}

\begin{remark}[Tightness of problem-dependent regret bounds]\label{rmk:tightness-problem-dependent-regret-bounds}
    The problem-dependent regret bounds are tight for a class of \MFMAB instances.
    For instances fulfilling the condition that \(\{m:\Delta_k\upbra{m}>0\} = \{M\}\) for all arm \(k\in\mathcal{K}\), then
    the problem-dependent regret lower bound matches the upper bound.
    This kind of instances covers a vast number of real world scenarios. Because in practice, the highest fidelity \(M\) is often defined as the only fidelity where the optimal arm can be distinguished from the suboptimal arms. For example, in neural architecture search, the process of increasing the training sample size (fidelity) stops when one architecture performs much better than others.
    % Therefore, this class of \MFMAB instances is general and contains many real-world scenarios.
\end{remark}

\begin{remark}[Relation to partial monitoring]
    Recall that, under our regret definition in~\eqref{eq:lambda1-regret}, the optimal policy is to pull the optimal arm \(1\) at the lowest fidelity, and any exploration at the highest fidelity results in nonzero regret.
    This clear separation of exploration and exploitation is similar to the hard case of partial monitoring~\cite{bartok2012adaptive,lattimore2019cleaning} where the decision maker needs to choose the globally observable actions with nonzero regret cost to do exploration and then, when exploiting, choose locally observable actions.
    % \wei{The above seems to be for adversarial partial monitoring. Stochastic partial monitoring also has similar bounds of \(\Theta(T^{2/3})\).
    % 	For a citation, please check my combinatorial partial monitoring paper in 2015 and find a citation in that paper.}
    The regret of the hard case of partial monitoring is \(\Theta(T^{{2}/{3}})\) where \(T\) is the decision round, which is also similar to our \(\tilde{O}(\Lambda^{{2}/{3}})\) regret bound.
\end{remark}

% !TeX root = ..\mf-mab.tex
\section{Future Directions}
We note that the \BAIC's exploration procedures in Algorithm~\ref{alg:fidelity-selection-procedure} require some prior knowledge of the top two arms' reward mean estimates \(\tilde{\mu}_1\upbra{M}\) and \(\tilde{\mu}_2\upbra{M}\) as input.
Although such prior knowledge is easy to access in many real world applications (see Remark~\ref{rmk:input-mu-1-2}), this is not a common assumption in bandits literature.
Without relying on this prior knowledge, we devise a third procedure \explore{C} in Appendix~\ref{sec:explore-C} which
starts from the lower fidelity and gradually increase fidelity when necessary.
% follows the rule of gradually increasing exploration fidelities. 
However, its cost complexity upper bound is incomparable to the lower bound in~\eqref{eq:sample-cost-lower-bound} and can be very large.
Therefore, one interesting future direction is to devise \BAIC algorithms without this prior knowledge but still enjoying good theoretical performance.

Another interesting future direction is to quantify the cost of identifying optimal fidelities and improve the current cost complexity lower bound in~\eqref{eq:sample-cost-lower-bound}.
Note that
the second term of cost complexity upper bound for \explore{A} in~\eqref{eq:sample-cost-upper-bound-A} has no correspondence in the lower bound, so does the additional factor \(\sum_{m\in\mathcal{M}} \lambda\upbra{m}/\lambda\upbra{1}\) of the bound for \explore{B} in~\eqref{eq:sample-cost-upper-bound-B}.
These two additional terms may correspond to the cost of finding the optimal fidelity \(m_k^*\) which is not accounted in the lower bound in Theorem~\ref{thm:sample-cost-lower-bound}.

\bibliography{mf-mab}
\bibliographystyle{plainnat}
%%%%%%%%%%%%%%%%%%%%%%%%%%%%%%%%%%%%%%%%%%%%%%%%%%%%%%%%%%%%

\clearpage
\appendix
\section*{Supplementary Materials}
% !TeX root = ..\mf-mab.tex
\section{Simulation Details of Figure~\ref{fig:explore-A-vs-B}}\label{sec:simulation-detail}

In Figure~\ref{fig:explore-A-vs-B}, we report the cost complexities of Algorithm~\ref{alg:lucb-framework} with \explore{A} and Algorithm~\ref{alg:lucb-framework} with \explore{B} (let \(\tilde{\mu}_1\upbra{M}=0.95\) and \(\tilde{\mu}_2\upbra{M}=0.75\)).
We set the confidence parameter \(\delta\) as \(0.05, 0.1, 0.15, 0.2, 0.25\) respectively in comparing the performance of both procedures.
For each simulation, we run \(100\) trials, plot their cost complexities' mean as markers and their deviation as shaded regions.
We present the parameters of \MFMAB instances of Figures~\ref{fig:explore-A-vs-B-1} and~\ref{fig:explore-A-vs-B-2} in Tables~\ref{tab:simulation-detail-1} and~\ref{tab:simulation-detail-2} respectively.
We note that it is more difficult to find the optimal fidelity \(m_k^*\) in the \MFMAB instances for Figure~\ref{fig:explore-A-vs-B-2} because (1) there are more fidelities choices in this instance than that of Figure~\ref{fig:explore-A-vs-B-1}; (2) the value of \(\tilde{\Delta}_k\upbra{m}/\sqrt{\lambda\upbra{m}}\) are closer in the second instance than that of Figure~\ref{fig:explore-A-vs-B-1}.

\begin{table}[htb]
    \centering
    \caption{Figure~\ref{fig:explore-A-vs-B-1}'s \MFMAB with \(K=5\) arms and \(M=3\) fidelities}\label{tab:simulation-detail-1}
    \begin{tabular}{cccccccc}
        \toprule
        Parameters & \(\mu_1\upbra{m}\) & \(\mu_2\upbra{m}\) & \(\mu_3\upbra{m}\) & \(\mu_4\upbra{m}\) & \(\mu_5\upbra{m}\) & \(\zeta\upbra{m}\) & \(\lambda\upbra{m}\)
        \\
        \midrule
        \(m=1\)
                   & \(0.70\)           & \(0.75\)
                   & \(0.50\)           & \(0.50\)           & \(0.30\)
                   & \(0.30\)           & \(1\)
        \\
        \(m=2\)    & \(0.80\)           & \(0.775\)
                   & \(0.60\)           & \(0.55\)           & \(0.45\)
                   & \(0.15\)           & \(1.1\)
        \\
        \(m=3\)    & \(0.90\)           & \(0.80\)
                   & \(0.70\)           & \(0.60\)           & \(0.50\)
                   & \(0\)              & \(1.2\)
        \\
        \bottomrule
    \end{tabular}
\end{table}

\begin{table}[htb]
    \centering
    \caption{Figure~\ref{fig:explore-A-vs-B-2}'s \MFMAB with \(K=5\) arms and \(M=5\) fidelities}\label{tab:simulation-detail-2}
    \begin{tabular}{cccccccc}
        \toprule
        Parameters & \(\mu_1\upbra{m}\) & \(\mu_2\upbra{m}\) & \(\mu_3\upbra{m}\) & \(\mu_4\upbra{m}\) & \(\mu_5\upbra{m}\) & \(\zeta\upbra{m}\) & \(\lambda\upbra{m}\)
        \\
        \midrule
        \(m=1\)
                   & \(0.83\)           & \(0.82\)
                   & \(0.76\)           & \(0.82\)           & \(0.70\)
                   & \(0.10\)           & \(1\)
        \\
        \(m=2\)    & \(0.84\)           & \(0.83\)
                   & \(0.80\)           & \(0.80\)           & \(0.72\)
                   & \(0.08\)           & \(1.1\)
        \\
        \(m=3\)    & \(0.85\)           & \(0.85\)
                   & \(0.80\)           & \(0.82\)           & \(0.74\)
                   & \(0.06\)           & \(1.2\)
        \\
        \(m=4\)    & \(0.85\)           & \(0.86\)
                   & \(0.80\)           & \(0.80\)           & \(0.76\)
                   & \(0.04\)           & \(1.3\)
        \\
        \(m=5\)    & \(0.90\)           & \(0.88\)
                   & \(0.86\)           & \(0.84\)           & \(0.80\)
                   & \(0\)              & \(1.4\)
        \\
        \bottomrule
    \end{tabular}
\end{table}

% !TeX root = ..\mf-mab.tex
\section{A Third Fidelity Selection Procedure: \explore{C}}\label{sec:explore-C}
% we devise three different mechanisms in \S\ref{subsubsec:three-subroutine}.
% \wei{This may need to  be changed. We already have two mechanisms in the main text. This appendix is only for the third mechanism, right? Also it should not be named Explore-A any more.
% }
Besides the \explore{A} and -B procedures, here we consider a third na\"ive and conservative idea for fidelity selection that one should start from low risk (cost), gradually increase the risk (cost) as the learning task needs,
and stop when finding the optimal arm.
To decide when to increase the fidelity for exploring an arm \(k\),
we use the arm's confidence radius \(\beta(N_{k,t}\upbra{m},t,\delta)\) at fidelity \(m\) as a measure of the amount of information left in this fidelity, and when the fidelity \(m\)'s confidence radius is less than the error upper bound \(\zeta\upbra{m}\) at this fidelity, we increase the fidelity by \(1\) for higher accuracy, or formally, the fidelity is selected as follows, \[
    m_{k,t}\gets \min\left\{ m \left| \beta(N_{k,t}\upbra{m},t,\delta)\ge \zeta^{(m)}\right.\right\}.
\]

\begin{algorithm}[h]
    \caption{\explore{C} Procedures}
    \begin{algorithmic}

        \Procedure{Explore-C}{$k$}
        % \State \textbf{Input:} $\zeta\upbra{1},\zeta\upbra{2},\dots,\zeta\upbra{M}$
        \State \(m_{k,t}\gets \min\left\{ m \left| \beta(N_{k,t}\upbra{m},t,\delta)\ge \zeta^{(m)}\right.\right\}\)
        \State Pull \(\displaystyle \left(k, m_{k,t}\right)\), observe reward, and update corresponding statistics
        \EndProcedure

    \end{algorithmic}
\end{algorithm}

% \paragraph{\explore{C}}
% The first procedure is based on an intuitive idea that is to start from low fidelity with low cost/accuracy and gradually increase to higher fidelities.

% Specially, we set \(\gamma\upbra{M}\) as zero so to assure that the set in right-hand-side always contains fidelity \(M\).
% This idea is similar to the regret minimization algorithm of~\citet{kandasamy2016multi,kandasamy2017multi}.
% We present the procedure in \explore{C} of Algorithm~\ref{alg:fidelity-selection-procedure}.

\begin{theorem}[Cost complexity upper bounds for Algorithm~\ref{alg:lucb-framework} with \explore{C}]\label{thm:cost-complexity-upper-bound-C}
    Given Assumption~\ref{asp:arm-2-is-good} and \(L \ge 4KM\), Algorithm~\ref{alg:lucb-framework} outputs the optimal arm with a probability at least \(1-\delta\).
    The cost complexity of Algorithm~\ref{alg:lucb-framework} with \emph{\explore{C}} are upper bounded as follows,
    \begin{equation}\label{eq:sample-cost-upper-bound-C}
        \begin{split}
            \E[\Lambda] &= O\left(H^\ddagger\log \left( \frac{L(H^\ddagger+Q)}{\lambda\upbra{1}\delta} \right) + Q\log \left( \frac{L(H^\ddagger+Q)}{\lambda\upbra{1}\delta} \right)\right),
        \end{split}
    \end{equation}
    where, letting \(m_k^\ddagger\) denote the smallest fidelity for arm \(k\) such that \(\Delta_k\upbra{m} > 2\zeta\upbra{m}\), or formally,\begin{equation*}
        m_k^\ddagger \coloneqq \min\{
        m:\Delta_k\upbra{m} > 2\zeta\upbra{m}
        \},
    \end{equation*}
    and we denote
    \[H^\ddagger \!\coloneqq \!\!\sum_{k\in\mathcal{K}}\!\!\frac{\lambda\upbra{m_k^\ddagger}}{(\Delta_k\upbra{m_k^\ddagger})^2},\quad
        Q \coloneqq \sum_{k\in \mathcal{K}}  \sum_{m = 1}^{m_k^\ddagger-1} \frac{\lambda\upbra{m}}{(\zeta\upbra{m})^2}.
    \]
\end{theorem}

\begin{remark}[\explore{C} \textit{vs.} \textsc{Explore}-A and -B]
    The \explore{C} procedure does not require additional knowledge as the other two. It is a one-size-fits-all option.
    If with some addition information of a specific scenario, e.g., the exact or approximated reward means of top two arms, one can use the \textsc{Explore}-A or B.
\end{remark}
% !TeX root = ..\mf-mab.tex
\section{Proofs for Best Arm Identification with Fixed Budget}

\subsection{Proof of Theorem~\ref{thm:sample-cost-lower-bound}}
\label{supapp:sample-cost-lower-bound}

\begin{lemma}[{\citet{kaufmann2016complexity}, Lemma 1}]\label{lma:sample-complexity-general-bound}
  Let \(\nu\) and \(\nu'\) be two bandit models with \(K\) arms such that for all \(k\), the distributions \(\nu_k\) and \(\nu'_k\) are mutually absolutely continuous. For any almost-surely finite stopping time \(\sigma\) with respect to the filtration \(\{\mathcal{F}_t\}_{t\in\mathbb{N}}\) where \(\mathcal{F}_t=\sigma(I_1,X_1,\dots,I_t,X_t)\),
  \[
    \sum_{k=1}^K \E_\nu[N_k(\sigma)]\KL(\nu_k, \nu_{k}') \ge \sup_{\mathcal{E}\in\mathcal{F}_\sigma} \kl(\P_\nu(\mathcal{E}), \P_{\nu'}(\mathcal{E})),
  \]
  where \(\kl(x,y)\) is the binary relative entropy.
\end{lemma}

In \MFMAB model, regarding each arm-fidelity \((k,m)\)-pair as an individual arm, we can extend Lemma~\ref{lma:sample-complexity-general-bound} to multi-fidelity case as follows,
\begin{equation}\label{eq:sample-complexity-general-bound}
  \sum_{k=1}^K\sum_{m=1}^M \E_\nu[N_k\upbra{m}(\sigma)]\KL(\nu_k\upbra{m}, {\nu'}_{k}\upbra{m}) \ge \sup_{\mathcal{E}\in\mathcal{F}_\sigma} \kl(\P_\nu(\mathcal{E}), \P_{\nu'}(\mathcal{E})).
\end{equation}

Next, we construct instances \(\nu\) and \(\nu'\).
We set the reward distributions \(\nu=(\nu_k\upbra{m})_{(k,m)\in\mathcal{K}\times\mathcal{M}}\) as Bernoulli and the reward means fulfill \(\mu_1\upbra{M}>\mu_2\upbra{M}\ge \mu_3\upbra{M}\ge \dots \ge \mu_K\upbra{M},\) where \(\mu_k\upbra{m} = \mathbb{E}_{X\sim \nu_k\upbra{m}}[X]\).
We let \({\nu'}_k\upbra{m}\) be the same to \(\nu_k\upbra{m}\) for all \(k\) and \(m\), except for that an arm \(\ell\neq 1\).
We set arm \(\ell\)'s reward means  on fidelities \(m\in\mathcal{M}_k\) to be \({\nu'}_\ell\upbra{m} = \nu_1\upbra{M}-\zeta\upbra{m} + \epsilon\).
So, in instance \(\nu'\), the optimal arm is \(\ell\) and its true reward mean \({\mu'}_\ell\upbra{M}\) is slightly greater than \(\mu_1\upbra{M}\).
This implies for the event \(\mathcal{E}=\{\text{output arm }1\}\) and any algorithm \(\pi\) that can find the optimal arm with a confidence \(1-\delta\), \(\P_{\nu,\pi}(\mathcal{E})\ge 1-\delta\) and \(\P_{\nu',\pi}(\mathcal{E}) \le \delta\).
Then, from \eqref{eq:sample-complexity-general-bound}, we have
\[
  \begin{split}
    \sum_{m\in\mathcal{M}_k} \E_\nu[N_\ell\upbra{m}(\sigma)]\KL(\nu_\ell\upbra{m}, {\nu'}_{\ell}\upbra{m})
    & \ge \sup_{\mathcal{E}\in\mathcal{F}_\sigma} \kl(\P_{\nu,\pi}(\mathcal{E}), \P_{\nu',\pi}(\mathcal{E})) \\
    & \ge \kl(1-\delta, \delta)\\
    & \ge \log \frac{1}{2.4\delta}.
  \end{split}
\]
We rewrite the above inequality as follows, \[
  \sum_{m\in\mathcal{M}_k} \lambda\upbra{m} \E_\nu[N_\ell\upbra{m}(\sigma)] \cdot \frac{\KL(\nu_\ell\upbra{m}, {\nu'}_{\ell}\upbra{m})}{\lambda\upbra{m}}
  \ge \log \frac{1}{2.4\delta}.
\]

Therefore, for the arm \(\ell\), our aim is to minimize its cost complexity with a constraint as follows,
\[\begin{split}
    & \min_{\mathbb{E}[N_\ell\upbra{m}],\forall m} \sum_{m=1}^M \lambda\upbra{m} \E_\nu[N_\ell\upbra{m}(\sigma)]\\
    \text{such that } & \sum_{m\in\mathcal{M}_k} \lambda\upbra{m} \E_\nu[N_\ell\upbra{m}(\sigma)] \cdot \frac{\KL(\nu_\ell\upbra{m}, {\nu'}_{\ell}\upbra{m})}{\lambda\upbra{m}}
    \ge \log \frac{1}{2.4\delta}.
  \end{split}
\]

Note that the above is a linear programming (LP) and its optimum is reached at one of its polyhedron constraint's vertex---only one \(\mathbb{E}[N_\ell\upbra{m}]\) is positive and all others are equal to zero.
\[\begin{split}
    \min_{\mathbb{E}[N_\ell\upbra{m}],\forall m}  \sum_{m=1}^M \lambda\upbra{m} \E_\nu[N_\ell\upbra{m}(\sigma)]
    & \overset{(a)}\ge\min_{m\in\mathcal{M}_k} \frac{\lambda\upbra{m}}{\KL(\nu_\ell\upbra{m}, {\nu'}_{\ell}\upbra{m})} \log \frac{1}{2.4\delta}\\
    & =  \min_{m\in\mathcal{M}_k} \frac{\lambda\upbra{m}}{\KL(\nu_\ell\upbra{m}, \nu_1\upbra{M}-\zeta\upbra{m} + \epsilon)} \log \frac{1}{2.4\delta}\\
    & \overset{(b)}\ge \min_{m\in\mathcal{M}_k} \frac{\lambda\upbra{m}}{(1+\varepsilon)\KL(\nu_\ell\upbra{m}, \nu_1\upbra{M}-\zeta\upbra{m})} \log \frac{1}{2.4\delta}
  \end{split}
\]
where the inequality (a) is due to the property of LP we mentioned above,
and the inequality (b) is because of the continuity of KL-divergence.

To bound the optimal arm \(1\)'s cost complexity, we use the same \(\nu\) as above and construct another instance \(\nu''\). The instance \(\nu''\)'s reward means are the same to \(\nu\) except for arm \(1\) whose reward means for fidelity \(m\in \mathcal{M}_1\) are set as \({\mu''}_1\upbra{m} = \mu_2\upbra{m}+\zeta\upbra{m}-\epsilon\). Then, with similar procedure as the above, we obtain
\[
  \begin{split}
    \min_{\mathbb{E}[N_1\upbra{m}],\forall m}  \sum_{m=1}^M \lambda\upbra{m} \E_\nu[N_1\upbra{m}(\sigma)]
    \ge \min_{m\in\mathcal{M}_1} \frac{\lambda\upbra{m}}{(1+\varepsilon)\KL(\nu_1\upbra{m}, \nu_2\upbra{M}+\zeta\upbra{m})} \log \frac{1}{2.4\delta}.
  \end{split}
\]

Summing up the above costs leads to the lower bound as follows, and letting the \(\epsilon\) goes to zeros concludes the proof.
\[
  \begin{split}
    &\E[\Lambda] \ge \\
    &\left( \min_{m\in\mathcal{M}_1} \frac{\lambda\upbra{m}}{(1+\varepsilon)\KL(\nu_{1}\upbra{m}, \nu_{2}\upbra{M} + \zeta\upbra{m})} + \sum_{k\neq 1}\min_{m\in\mathcal{M}_k}  \frac{\lambda\upbra{m}}{(1+\varepsilon)\KL(\nu_k\upbra{m}, \nu_1\upbra{M}-\zeta\upbra{m})} \right) \log \frac{1}{2.4\delta}.
  \end{split}
\]

\subsection{Proof of Theorem~\ref{thm:sample-cost-upper-bound}}
\label{subapp:sample-cost-upper-bound-proof}

\textbf{Notation.} Denote the threshold \(c = \frac{\mu_1\upbra{M} + \mu_2\upbra{M}}{2}\) as the average of the optimal and best suboptimal arms' reward means.
Denote \(\mathcal{A}_t\coloneqq \{k\in\mathcal{K}: \texttt{LCB}_{k,t} > c\}\)
and \(\mathcal{B}_t\coloneqq \{k\in\mathcal{K}: \texttt{UCB}_{k,t} < c\}\)
as the above and below sets which respectively contain arms whose rewards are clearly higher or lower than the threshold with high probability, and let \(\mathcal{C}_t\coloneqq \mathcal{K}\setminus (\mathcal{A}_t \cup \mathcal{B}_t)\) as the complement of both sets' union. Then, we define two events as follows
\[
  \begin{split}
    \texttt{TERM}_t &\coloneqq \{\texttt{LCB}_{\ell_t,t} > \texttt{UCB}_{u_t,t}\},\\
    \texttt{CROS}_t &\coloneqq \{\exists k\neq 1: k\in\mathcal{A}_t  \} \cup \{1\in\mathcal{B}_t\}.
  \end{split}
\]

The \(\texttt{TERM}_t\) event corresponds to the complement of the main while loop condition in the LUCB algorithm.
When the \(\texttt{TERM}_t\) event happens, the LUCB algorithm terminates.
The \(\texttt{CROS}_t\) event means there exists a suboptimal arm whose \(\texttt{LCB}_{k,t}\) is greater than \(c\) or that the optimal arm \(1\)'s \(\texttt{UCB}_{1,t}\) is less than \(c\), both of which means that at least one arm's reward mean  confidence interval incorrectly crosses the threshold \(c\).

\paragraph{Step 1. Prove \(\lnot \texttt{TERM}_t \cap \lnot \texttt{CROS}_t \Longrightarrow (\ell_t\in\mathcal{C}_t) \cup (u_t\in\mathcal{C}_t)\).}
We show this statement by contradiction case by case. That is, the negation of \((\ell_t\in\mathcal{C}_t) \cup (u_t\in\mathcal{C}_t)\) cannot happen when \(\lnot \texttt{TERM}_t \cap \lnot \texttt{CROS}_t\).
\[
  \begin{split}
    \textbf{Case 1: } & (\ell_t \in \mathcal{A}_t) \cap (u_t \in \mathcal{A}_t) \cap \lnot \texttt{TERM}_t \\
    \Longrightarrow & (\ell_t \in \mathcal{A}_t) \cap (u_t \in \mathcal{A}_t)
    \Longrightarrow  \abs{\mathcal{A}_t} \ge 2
    \Longrightarrow  \exists k\neq 1: k\in\mathcal{A}_t
    \Longrightarrow  \texttt{CROS}_t,\\
    \textbf{Case 2: } & (\ell_t \in \mathcal{B}_t) \cap (u_t \in \mathcal{A}_t) \cap \lnot \texttt{TERM}_t \\
    \Longrightarrow & \texttt{UCB}_{\ell_t,t} < c < \texttt{LCB}_{u_t,t} < \texttt{UCB}_{u_t,t}
    \Longrightarrow \emptyset \,(\text{contradicts the selection of }\ell_t\text{ and }u_t),\\
    \textbf{Case 3: } & (\ell_t \in \mathcal{A}_t) \cap (u_t \in \mathcal{B}_t) \cap \lnot \texttt{TERM}_t\\
    \Longrightarrow & \{\texttt{LCB}_{\ell_t,t} > c > \texttt{UCB}_{u_t,t}\} \cap \lnot \texttt{TERM}_t
    \Longrightarrow  \emptyset,\\
    \textbf{Case 4: } & (\ell_t \in \mathcal{B}_t) \cap (u_t \in \mathcal{B}_t) \cap \lnot \texttt{TERM}_t\\
    \Longrightarrow & (\ell_t \in \mathcal{B}_t) \cap (u_t \in \mathcal{B}_t)
    \Longrightarrow \abs{\mathcal{B}_t} = K
    \Longrightarrow  1\in\mathcal{B}_t
    \Longrightarrow \texttt{CROS}_t.
  \end{split}
\]

\paragraph{Step 2. Prove \(\P\left( \texttt{CROS}_t \right) \le \frac{KM\delta}{Lt^3}\).}
For any suboptimal arm \(k \neq 1\), we bound the probability that the arm \(k\) is in \(\mathcal{A}_t\) as follows,
\[
  \begin{split}
    \P(k\in\mathcal{A}_t)  &= \P(\texttt{LCB}_{k,t} > c)
    = \P\left(\max_{m\in\mathcal{M}} \texttt{LCB}_{k,t}\upbra{m} > c\right)
    \le \sum_{m\in\mathcal{M}}\P\left(\texttt{LCB}_{k,t}\upbra{m} > c\right)\\
    &= \sum_{m\in\mathcal{M}}\P\left(\hat{\mu}_{k,t}\upbra{m} - \zeta\upbra{m} - \beta(N_{k,t}\upbra{m}, t) > c\right)\\
    &= \sum_{m\in\mathcal{M}}\P\left(\hat{\mu}_{k,t}\upbra{m} - \mu_k\upbra{m} + (\mu_k\upbra{m} - \zeta\upbra{m} -c) > \beta(N_{k,t}\upbra{m}, t)\right)\\
    &\overset{(a)}\le \sum_{m\in\mathcal{M}}\P\left(\hat{\mu}_{k,t}\upbra{m} - \mu_k\upbra{m} >  \beta(N_{k,t}\upbra{m}, t)\right) \le \sum_{m\in\mathcal{M}} \sum_{n=1}^t \P(\hat{\mu}_{k,t}\upbra{m} - \mu_k\upbra{m} >  \beta(n, t))\\
    &\le \sum_{m\in\mathcal{M}} \sum_{n=1}^t \exp(-n(\beta(n,t))^2) = \sum_{m\in\mathcal{M}}\sum_{n=1}^t \frac{\delta}{Lt^4}\\
    &\le \frac{M\delta}{Lt^3},
  \end{split}
\]
where the inequality (a) is due to that \( \mu_k\upbra{m} - \zeta\upbra{m} \le \mu_k\upbra{M} <  c\).
With similar derivation, we have \(\P(1\in\mathcal{B}_t) \le \frac{M\delta}{Lt^3}.\)
Noticing that \(\P(\texttt{CROS}_t) \le \sum_{k\neq 1} \P(k\in\mathcal{A}_t) + \P(1\in \mathcal{B}_t)\), we have \(\P(\texttt{CROS}_t) \le \frac{KM\delta}{Lt^3}\).

\paragraph{Step 3. Prove \(\P\left( \exists k \in\mathcal{K}: (N_{k,t}\upbra{m^*_k} > 16N^*_{k,t}) \cap (k\in \texttt{Mid}_t) \right) \le \frac{16\delta\sum_{k\in\mathcal{K}}\Delta_k^{-2}}{L t^4}\), where \(N_{k,t}^* \coloneqq \frac{\log (Lt^4/\delta)}{(\Delta_k\upbra{\tilde{m}_k^*})^2}\).}
For any fixed suboptimal arm \(k \neq 1\) (with \(\mu_k\upbra{M} < c\)), we have
\[
  \begin{split}
    &\quad \P\left((N_{k,t}\upbra{m^*_k} > 16N_{k,t}^*) \cap (k\in \texttt{Mid}_t) \right)\\
    & = \P\left((N_{k,t}\upbra{m^*_k} > 16N_{k,t}^*) \cap (k\not\in \mathcal{A}_t \cup \mathcal{B}_t) \right)\\
    & \le  \P\left((N_{k,t}\upbra{m^*_k} > 16N_{k,t}^*) \cap (\texttt{UCB}_{k,t} > c) \right)\\
    & = \P\left((N_{k,t}\upbra{m^*_k} > 16N_{k,t}^*) \cap \left(\min_{m\in\mathcal{M}}\hat{\mu}_{k,t}\upbra{m} + \zeta\upbra{m} + \beta(N_{k,t}\upbra{m}, t) > c\right) \right)\\
    & \le  \P\left((N_{k,t}\upbra{m^*_k} > 16N_{k,t}^*) \cap (\hat{\mu}_{k,t}\upbra{m^*_k} + \zeta\upbra{m^*_k} + \beta(N_{k,t}\upbra{m^*_k}, t) > c) \right)\\
    & \le  \P\left((N_{k,t}\upbra{m^*_k} > 16N_{k,t}^*) \cap (\hat{\mu}_{k,t}\upbra{m^*_k} - \mu_k\upbra{m^*_k} > (c - \mu_k\upbra{m^*_k} - \zeta\upbra{m^*_k}) - \beta(N_{k,t}\upbra{m^*_k}, t)) \right)\\
    & \overset{(a)}\le  \P\left((N_{k,t}\upbra{m^*_k} > 16N_{k,t}^*) \cap \left(\hat{\mu}_{k,t}\upbra{m^*_k} - \mu_k\upbra{m^*_k} > \frac{\Delta_k\upbra{\tilde{m}_k^*}}{2} - \beta(N_{k,t}\upbra{\tilde{m}_k^*}, t)\right) \right)\\
    & \overset{(b)}\le \sum_{\tau >16N_{k,t}^*}\P\left(\hat{\mu}_{k,t(\tau)}\upbra{m^*_k} - \mu_k\upbra{m^*_k} > \frac{\Delta_k\upbra{\tilde{m}_k^*}}{4}\right)\quad \left(\text{denote }\hat{\mu}_{k,t(\tau)}\upbra{m^*_k} \text{ as the empirical mean of }\tau\text{ observations}\right)\\
    & \le \sum_{\tau >16N_{k,t}^*}\exp\left( -\frac{\tau(\Delta_k\upbra{\tilde{m}_k^*})^2}{16} \right)
    \le \int_{\tau>16N_{k,t}^*}\exp\left( -\frac{\tau(\Delta_k\upbra{\tilde{m}_k^*})^2}{16} \right) d\tau
    \le \frac{16\delta}{(\Delta_k\upbra{\tilde{m}_k^*})^2 L t^4},
  \end{split}
\]
where inequality (a) is due to \eqref{eq:arm-2-is-good}
and inequality (b) is due to \(\beta(\tau,t) < \frac{\Delta_k}{4}\) for \(\tau > 16N_{k,t}^*.\)

From Step 3, we obtain that the following equation holds with high probability, \begin{equation}\label{eq:m-star-sample-times-upper-bound}
  N_{k,t}\upbra{\tilde{m}_k^*} \le \frac{16}{(\Delta_k\upbra{\tilde{m}_k^*})^2}\log \left( \frac{Lt^4}{\delta} \right) \le \frac{64}{(\Delta_k\upbra{\tilde{m}_k^*})^2}\log \left( \frac{Lt}{\delta} \right).
\end{equation}

Next, we respectively present the cost complexity upper bounds for different fidelity selection procedures in Algorithm~\ref{alg:fidelity-selection-procedure}.

\subsubsection{Proof for \explore{A}'s Upper Bound}

\paragraph{Step 4 for \explore{A}: prove that if the small probability events of Steps 2 and 3 do not happen, then the algorithm terminates with a high probability when \(\Lambda\) is large.}

\begin{lemma}\label{lma:f-ucb-property}
  Give reward means \(\mu_1\upbra{M}\) and \(\mu_2\upbra{M}\).
  For a fixed arm \(k\), there exist \(\bar{N}_{k,t}\) and \(\alpha_k>0\) such that when \(N_{k,t} > \bar{N}_{k,t}\), \(N_{k,t} < 2 N_{k,t}\upbra{\tilde{m}_k^*},\)
  the number of times of pulling this arm \(k\) at fidelities \(m \,(\neq \tilde{m}_k^*)\) is \(O(\log (\log N_{k,t}))\), or formally, \begin{equation}
    N_{k,t}\upbra{m} \le \frac{8}{\lambda\upbra{m}}\left( \frac{\Delta_k\upbra{\tilde{m}_k^*}}{\sqrt{\lambda\upbra{\tilde{m}_k^*}}} - \frac{\Delta_k\upbra{m}}{\sqrt{\lambda\upbra{m}}} \right)^{-2} \log N_{k,t},\,\forall m\neq \tilde{m}_k^*.
  \end{equation}
\end{lemma}

Combine Lemma~\ref{lma:f-ucb-property} with \eqref{eq:m-star-sample-times-upper-bound} in Step 3, we have, for any arm \(k\) and fidelity \(m\neq \tilde{m}_k^*\): \begin{equation}
  \label{eq:m-nonstar-sample-times-upper-bound}
  N_{k,t}\upbra{m} \le \frac{8}{\lambda\upbra{m}}\left( \frac{\Delta_k\upbra{\tilde{m}_k^*}}{\sqrt{\lambda\upbra{\tilde{m}_k^*}}} - \frac{\Delta_k\upbra{m}}{\sqrt{\lambda\upbra{m}}} \right)^{-2} \log \left(  \frac{128}{(\Delta_k\upbra{\tilde{m}_k^*})^2} \log \left( \frac{Lt}{\delta} \right) \right)
\end{equation}

Next, we can upper bound the total cost of the LUCB algorithm (before it terminating) via \eqref{eq:m-star-sample-times-upper-bound} and \eqref{eq:m-nonstar-sample-times-upper-bound}.
Specially, we show it is impossible for \(\Lambda = C\left( H\log\frac{L(G+H)}{\lambda\upbra{1}\delta} + G\log\log\frac{L(G+H)}{\lambda\upbra{1}\delta} \right)\) via contradiction.
Suppose \(\Lambda = C\left( H\log\frac{L(G+H)}{\lambda\upbra{1}\delta} + G\log\log\frac{L(G+H)}{\lambda\upbra{1}\delta} \right)\), we have the following,
\[
  \begin{split}
    \E[\Lambda] &\le \sum_{k\in\mathcal{K}}\sum_{m\in\mathcal{M}} \lambda\upbra{m}N_{k,t}\upbra{m}\\
    &\le \sum_{k\in\mathcal{K}} \lambda\upbra{\tilde{m}_k^*}N_{k,t}\upbra{\tilde{m}_k^*} + \sum_{k\in\mathcal{K}}\sum_{m\neq \tilde{m}_k^*} \lambda\upbra{m}N_{k,t}\upbra{m}\\
    &\overset{(a)}\le 64 H \log\frac{Lt}{\delta} + 8 G\log\log\frac{Lt}{\delta} + G\log (128 H)\\
    &\overset{(b)}\le 64 H \log\frac{L\Lambda}{\lambda\upbra{1}\delta} + 8 G\log\log\frac{L\Lambda}{\lambda\upbra{1}\delta} + G\log (128 H)\\
    &\overset{(c)}= 64H \log\left( \frac{L}{\lambda\upbra{1}\delta} C \left( H\log \frac{L(G+H)}{\lambda\upbra{1}\delta} + G\log\log \frac{L(G+H)}{\lambda\upbra{1}\delta}  \right) \right)\\
    &\quad + 8G \log\log\left( \frac{L}{\lambda\upbra{1}\delta} C \left( H\log \frac{L(G+H)}{\lambda\upbra{1}\delta} + G\log\log \frac{L(G+H)}{\lambda\upbra{1}\delta}   \right) \right) + G\log(128H)\\
    &\overset{(d)}\le 128(2+\log C) \left( H\log \frac{L(G+H)}{\lambda\upbra{1}\delta} + G\log\log \frac{L(G+H)}{\lambda\upbra{1}\delta}  \right)\\
    &\overset{(e)}< C \left( H\log \frac{L(G+H)}{\lambda\upbra{1}\delta} + G\log\log \frac{L(G+H)}{\lambda\upbra{1}\delta}  \right),
  \end{split}
\]
where
the inequality (a) is due to \eqref{eq:m-star-sample-times-upper-bound} and \eqref{eq:m-nonstar-sample-times-upper-bound},
the inequality (b) is because \(t\le\frac{\Lambda}{\lambda\upbra{1}}\),
the inequality (c) is by the supposition,
the inequality (d) is by separately bounding the above first two terms via \eqref{eq:bound-Lambda-term-1} and \eqref{eq:bound-Lambda-term-2} in the following, and
the inequality (e) holds for \(C>1200\).
This above inequality contradicts the supposition, and, therefore, we conclude the cost complexity upper bound proof for \explore{A}. Similar proof also holds for \explore{B} by replacing \(\tilde{m}_k^*\) with \(m_k^\dagger\).

Next, we provide the upper bounds used in the inequality (d) above:
\begin{equation}\label{eq:bound-Lambda-term-1}
  \begin{split}
    &\quad H \log\left( \frac{L}{\lambda\upbra{1}\delta} C \left( H\log \frac{L(G+H)}{\lambda\upbra{1}\delta} + G\log\log \frac{L(G+H)}{\lambda\upbra{1}\delta}  \right) \right) \\
    &\le H \log\left( \frac{L}{\lambda\upbra{1}\delta} C H\log \frac{L(G+H)}{\lambda\upbra{1}\delta} \right)
    + H \log\left( \frac{L}{\lambda\upbra{1}\delta} C G\log\log \frac{L(G+H)}{\lambda\upbra{1}\delta}  \right)\\
    &\le H\log C + H\log\left( \frac{LH}{\lambda\upbra{1}\delta} \log \frac{L(G+H)}{\lambda\upbra{1}\delta} \right)
    + H\log C + H\log \left( \frac{LG}{\lambda\upbra{1}\delta} \log\log  \frac{L(G+H)}{\lambda\upbra{1}\delta} \right)\\
    &\le 2H\log C + 4H\log \frac{L(G+H)}{\lambda\upbra{1}\delta} \\
    &\le (4+2\log C) H\log \frac{L(G+H)}{\lambda\upbra{1}\delta},
  \end{split}
\end{equation}
and
\begin{equation}\label{eq:bound-Lambda-term-2}
  \begin{split}
    &\quad G \log\log\left( \frac{L}{\lambda\upbra{1}\delta} C \left( H\log \frac{L(G+H)}{\lambda\upbra{1}\delta} + G\log\log \frac{L(G+H)}{\lambda\upbra{1}\delta}   \right) \right)\\
    &\le G\log\log \left( \frac{L}{\lambda\upbra{1}\delta} CH \log \frac{L(G+H)}{\lambda\upbra{1}\delta} \right) + G \log\log\left( \frac{L}{\lambda\upbra{1}\delta} C  G\log\log \frac{L(G+H)}{\lambda\upbra{1}\delta}  \right) \\
    &\le G\log\log C + G\log\log \left( \frac{LH}{\lambda\upbra{1}\delta} \log \frac{L(G+H)}{\lambda\upbra{1}\delta} \right) \\
    &\quad + G\log\log C + G\log\log\left( \frac{LG}{\lambda\upbra{1}\delta}\log\log \frac{L(G+H)}{\lambda\upbra{1}\delta} \right)\\
    &\le 2G\log\log C + 4G\log\log \frac{L(G+H)}{\lambda\upbra{1}\delta} \\
    &\le (4+2\log\log C) G\log\log\frac{L(G+H)}{\lambda\upbra{1}\delta}.
  \end{split}
\end{equation}

\begin{proof}[Proof of Lemma~\ref{lma:f-ucb-property}]
  \begin{claim}
    For any fixed \(m\neq \tilde{m}_k^*\), if the following equation holds, then the algorithm will not pull arm \(k\) at fidelity \(m\) with high probability.
    \[
      N_{k,t}\upbra{m} >  \frac{8}{\lambda\upbra{m}}\left( \frac{\Delta_k\upbra{\tilde{m}_k^*}}{\sqrt{\lambda\upbra{\tilde{m}_k^*}}} - \frac{\Delta_k\upbra{m}}{\sqrt{\lambda\upbra{m}}} \right)^{-2} \log N_{k,t}.
    \]
  \end{claim}

  \[
    \begin{split}
      \texttt{f-UCB}_{u_t,t}\upbra{m}(\mu_1\upbra{M}) &= \frac{1}{\sqrt{\lambda\upbra{m}}}\left(
      \mu_1\upbra{M} - \hat{\mu}_{u_t,t}\upbra{m} - \zeta\upbra{m}
      +
      \sqrt{\frac{2\log N_{u_t,t}}{N_{u_t,t}\upbra{m}}}
      \right)\\
      &\overset{(a)}\le \frac{1}{\sqrt{\lambda\upbra{m}}}
      \left(
      \mu_1\upbra{M} - {\mu}_{u_t}\upbra{m} - \zeta\upbra{m}
      +
      2\sqrt{\frac{2\log N_{u_t,t}}{N_{u_t,t}\upbra{m}}}
      \right) \\
      &\overset{(b)}\le \frac{1}{\sqrt{\lambda\upbra{m_{u_t}^*}}}
      \left(
      \mu_1\upbra{M} - {\mu}_{u_t}\upbra{m_{u_t}^*} - \zeta\upbra{m_{u_t}^*}
      \right)\\
      &\overset{(c)}\le \frac{1}{\sqrt{\lambda\upbra{\tilde{m}_k^*}}}
      \left(
      \mu_1\upbra{M} - \hat{\mu}_{u_t,t}\upbra{m_{u_t}^*} - \zeta\upbra{m_{u_t}^*}
      +
      \sqrt{\frac{2\log N_{u_t,t}}{N_{u_t,t}\upbra{\tilde{m}_k^*}}}
      \right)\\
      &=\texttt{f-UCB}_{u_t,t}\upbra{m_{u_t}^*}(\mu_1\upbra{M}),
    \end{split}
  \]
  where the inequalities (a) and (c) hold with a probability at least \(1-\frac{1}{(N_{k,t})^2}\) respectively (by Hoeffding's inequality),
  and the inequality (b) holds due to the equation in the claim.
  % \textcolor{blue}{We note that the number of times that the inequality~\eqref{eq:m-nonstar-sample-times-upper-bound} does not hold is finite. Consequently, we can say that~\eqref{eq:m-nonstar-sample-times-upper-bound} holds for sure when \(t\) is large (how large is large?). }

  \subsubsection{Proof for \explore{B}'s Upper Bound}
  As \explore{B} of Algorithm~\ref{alg:fidelity-selection-procedure} also employs the LUCB framework, it shares the first three steps of the proof for Theorem~\ref{thm:sample-cost-upper-bound} in Appendix~\ref{subapp:sample-cost-upper-bound-proof}.
  Hence, in this part, we focus on the proof of the final cost complexity upper bound.

  \begin{lemma}\label{lma:good-fidelity-guarantee}
    If the condition in Line~\ref{line:fidelity-commit-condition} holds, then the committed fidelity \(\hat{m}_k^*\) fulfills the following inequality:
    \begin{equation}\label{eq:property-of-committed-fidelity}
      2\cdot\frac{\Delta_k\upbra{\hat{m}_k^*}}{\sqrt{\lambda\upbra{\hat{m}_k^*}}} \ge \frac{\Delta_k\upbra{\tilde{m}_k^*}}{\sqrt{\lambda\upbra{\tilde{m}_k^*}}}.
    \end{equation}
  \end{lemma}

  \begin{proof}[Proof of Lemma~\ref{lma:good-fidelity-guarantee}]
    \eqref{eq:property-of-committed-fidelity} is proved as follows,
    \[
      \begin{split}
        \frac{{{\Delta}_k\upbra{\tilde{m}_k^*}}/{\sqrt{\lambda\upbra{\tilde{m}_k^*}}}}{{{\Delta}_k\upbra{\hat{m}_k^*}}/{\sqrt{\lambda\upbra{\hat{m}_k^*}}}}
        & \le
        \frac{{(\hat{\mu}_*\upbra{M} - (\hat{\mu}_k\upbra{\tilde{m}_k^*} + \zeta\upbra{\tilde{m}_k^*}))}/{\sqrt{\lambda\upbra{\tilde{m}_k^*}}} + \sqrt{{\log(2KM/\delta)}/{\lambda\upbra{1}N_{k,t}\upbra{m}}}}{{(\hat{\mu}_*\upbra{M} - (\hat{\mu}_k\upbra{\hat{m}_k^*} + \zeta\upbra{\hat{m}_k^*}))}/{\sqrt{\lambda\upbra{\hat{m}_k^*}}} - \sqrt{{\log(2KM/\delta)}/{\lambda\upbra{1}N_{k,t}\upbra{m}}}}\\
        & \overset{(a)}\le
        \frac{{(\hat{\mu}_*\upbra{M} - (\hat{\mu}_k\upbra{\hat{m}_k^*} + \zeta\upbra{\hat{m}_k^*}))}/{\sqrt{\lambda\upbra{\hat{m}_k^*}}} + \sqrt{{\log(2KM/\delta)}/{\lambda\upbra{1}N_{k,t}\upbra{m}}}}{{(\hat{\mu}_*\upbra{M} - (\hat{\mu}_k\upbra{\hat{m}_k^*} + \zeta\upbra{\hat{m}_k^*}))}/{\sqrt{\lambda\upbra{\hat{m}_k^*}}} - \sqrt{{\log(2KM/\delta)}/{\lambda\upbra{1}N_{k,t}\upbra{m}}}}\\
        & \le 1 + \frac{2\sqrt{{\log(2KM/\delta)}/{\lambda\upbra{1}N_{k,t}\upbra{m}}}}{{(\hat{\mu}_*\upbra{M} - (\hat{\mu}_k\upbra{\hat{m}_k^*} + \zeta\upbra{\hat{m}_k^*}))}/{\sqrt{\lambda\upbra{\hat{m}_k^*}}} - \sqrt{{\log(2KM/\delta)}/{\lambda\upbra{1}N_{k,t}\upbra{m}}}}\\
        &\overset{(b)}\le 1 + \frac{2\sqrt{{\log(2KM/\delta)}/{\lambda\upbra{1}N_{k,t}\upbra{m}}}}{2\sqrt{{\log(2KM/\delta)}/{\lambda\upbra{1}N_{k,t}\upbra{m}}}} = 2,
      \end{split}
    \]
    where inequality (a) is due to the definition of \(\hat{m}_k^*\), and inequality (b) is due to the condition in Line~\ref{line:fidelity-commit-condition}.
  \end{proof}

  We next upper bound the number of times of \(N_{k,t}\upbra{m}\) that guarantees that the condition in Line~\ref{line:fidelity-commit-condition} is true.
  Let us consider the case of exploring arm \(u_t\).
  \[
    \begin{split}
      \max_{m\in\mathcal{M}} \frac{\hat{\Delta}_{k,t}\upbra{m}}{\sqrt{\lambda\upbra{m}}} &\ge \frac{\hat{\mu}_{{k}_*}\upbra{M} - (\hat{\mu}_{k}\upbra{\tilde{m}_k^*} + \zeta\upbra{\tilde{m}_k^*})}{\sqrt{\lambda\upbra{\tilde{m}_k^*}}}\\
      &\overset{(a)}\ge \frac{\hat{\mu}_{{k}_*}\upbra{M} - ({\mu}_{k}\upbra{\tilde{m}_k^*} + \zeta\upbra{\tilde{m}_k^*})}{\sqrt{\lambda\upbra{\tilde{m}_k^*}}}
      - \sqrt{\frac{\log(2KM/\delta)}{\lambda\upbra{1}N_{k,t}\upbra{m}}}\\
      &\overset{(b)}\ge \frac{\Delta_k\upbra{\tilde{m}_k^*}}{\sqrt{\lambda\upbra{\tilde{m}_k^*}}}
      - \sqrt{\frac{\log(2KM/\delta)}{\lambda\upbra{1}N_{k,t}\upbra{m}}},
    \end{split}
  \]
  where inequality (a) is because that \(\hat{\mu}_k\upbra{\tilde{m}_k^*} \le \mu_k\upbra{\tilde{m}_k^*} + \sqrt{\frac{\log(2KM/\delta)}{N_{k,t}\upbra{m}}}\) with a probability of at least \(1-\delta/2KM\) (therefore, with the union bound over all arm-fidelity pairs, the total failure probability of \textsc{Explore} is upper bounded by \(\delta/2\)),
  and inequality (b) is because \(\hat{\mu}_{{k}_*}\upbra{M} - ({\mu}_{k}\upbra{\tilde{m}_k^*} + \zeta\upbra{\tilde{m}_k^*}) \ge {\mu}_{{k}_*}\upbra{M} - ({\mu}_{k}\upbra{\tilde{m}_k^*} + \zeta\upbra{\tilde{m}_k^*})
  = \Delta_k\upbra{\tilde{m}_k^*}\).

  To make the condition in Line~\ref{line:fidelity-commit-condition} hold, with the above inequality, we need
  \[
    \frac{\Delta_k\upbra{\tilde{m}_k^*}}{\sqrt{\lambda\upbra{\tilde{m}_k^*}}}
    - \sqrt{\frac{\log(2KM/\delta)}{\lambda\upbra{1}N_{k,t}\upbra{m}}} \ge 3\sqrt{\frac{\log(2KM/\delta)}{\lambda\upbra{1}N_{k,t}\upbra{m}}},
  \]
  which, after rearrangement, becomes   \[
    N_{k,t}\upbra{m} > \frac{16\lambda\upbra{\tilde{m}_k^*}}{(\Delta_k\upbra{\tilde{m}_k^*})^2} \frac{\log(2KM/\delta)}{\lambda\upbra{1}}.
  \]
  It means that if the above inequality holds, than the condition in Line~\ref{line:fidelity-commit-condition} must hold.
  That is, except for the committed fidelity \(\hat{m}_k^*\), we have
  \[
    N_{k,t}\upbra{m} \le  \frac{16\lambda\upbra{\tilde{m}_k^*}}{(\Delta_k\upbra{\tilde{m}_k^*})^2} \frac{\log(2KM/\delta)}{\lambda\upbra{1}},
    \text{ for any other fidelities } m\neq \hat{m}_k^*.
  \]

  For another thing,
  \eqref{eq:m-star-sample-times-upper-bound} of LUCB's proof guarantees that for the selected fidelity \(\hat{m}_k^*\), the number of pulling times is upper bounded as follows,
  \[
    N_{k,t}\upbra{\hat{m}_k^*} \le  \frac{64}{(\Delta_k\upbra{\hat{m}_k^*})^2}\log \left( \frac{Lt}{\delta} \right).
  \]

  Then, we upper bound the total budget of the algorithm as follows, \[
    \begin{split}
      \Lambda &= \sum_{k\in\mathcal{K}} \sum_{m\in\mathcal{M}} \lambda\upbra{m} N_{k,t}\upbra{m} \\
      &\le \sum_{k\in\mathcal{K}}  \frac{64\lambda\upbra{\hat{m}_k^*}}{(\Delta_k\upbra{\hat{m}_k^*})^2}\log \left( \frac{Lt}{\delta} \right)
      + \sum_{k\in\mathcal{K}} \sum_{m\neq \hat{m}_k^*} \frac{16\lambda\upbra{m}\lambda\upbra{\tilde{m}_k^*}}{(\Delta_k\upbra{\tilde{m}_k^*})^2\lambda\upbra{1}} \log\left(\frac{2KM}{\delta}\right)\\
      & \overset{(a)}\le \sum_{k\in\mathcal{K}}  \frac{256\lambda\upbra{\tilde{m}_k^*}}{(\Delta_k\upbra{\tilde{m}_k^*})^2}\log \left( \frac{Lt}{\delta} \right)
      + \sum_{k\in\mathcal{K}} \sum_{m\neq \hat{m}_k^*} \frac{16\lambda\upbra{m}\lambda\upbra{\tilde{m}_k^*}}{(\Delta_k\upbra{\tilde{m}_k^*})^2\lambda\upbra{1}} \log\left(\frac{2KM}{\delta}\right)\\
      & \le \left(\sum_{k\in\mathcal{K}} \frac{256\lambda\upbra{\tilde{m}_k^*}}{(\Delta_k\upbra{\tilde{m}_k^*})^2} + \sum_{k\in\mathcal{K}} \sum_{m\neq \hat{m}_k^*} \frac{16\lambda\upbra{m}\lambda\upbra{\tilde{m}_k^*}}{(\Delta_k\upbra{\tilde{m}_k^*})^2\lambda\upbra{1}}  \right)  \log \left( \frac{Lt}{\delta} \right) \\
      & \le \sum_{m\in\mathcal{M}}\frac{\lambda\upbra{m}}{\lambda\upbra{1}}  \cdot \sum_{k\in\mathcal{K}}  \frac{256\lambda\upbra{\tilde{m}_k^*}}{(\Delta_k\upbra{\tilde{m}_k^*})^2} \log \left( \frac{Lt}{\delta} \right)\\
      & \le \sum_{m\in\mathcal{M}}\frac{\lambda\upbra{m}}{\lambda\upbra{1}}  \cdot \sum_{k\in\mathcal{K}} \frac{256\lambda\upbra{\tilde{m}_k^*}}{(\Delta_k\upbra{\tilde{m}_k^*})^2} \log \left( \frac{L\Lambda}{\delta\lambda\upbra{1}} \right)\\
      & \overset{(b)}\le
      \sum_{m\in\mathcal{M}}\frac{\lambda\upbra{m}}{\lambda\upbra{1}}  \cdot \sum_{k\in\mathcal{K}} \frac{1024\lambda\upbra{\tilde{m}_k^*}}{(\Delta_k\upbra{\tilde{m}_k^*})^2} \log \left( \sum_{m\in\mathcal{M}}\frac{\lambda\upbra{m}}{\lambda\upbra{1}}  \cdot \sum_{k\in\mathcal{K}} \frac{256\lambda\upbra{\tilde{m}_k^*}}{(\Delta_k\upbra{\tilde{m}_k^*})^2} \frac{L}{\lambda\upbra{1}\delta} \right)\\
      &\le O\left( \sum_{m\in\mathcal{M}}\frac{\lambda\upbra{m}}{\lambda\upbra{1}}  \cdot \tilde{H} \log \left( \sum_{m\in\mathcal{M}}\frac{\lambda\upbra{m}}{\lambda\upbra{1}}  \cdot \tilde{H} \cdot \frac{L}{\lambda\upbra{1}\delta} \right) \right),
    \end{split}
  \]
  where inequality (a) is due to~\eqref{eq:property-of-committed-fidelity},
  inequality (b) is due to that \(\Lambda \le A\log(B\Lambda)\Longrightarrow \Lambda \le 4A\log(AB\Lambda)\).

  \subsubsection{Proof for \explore{C}'s Upper Bound in Theorem~\ref{thm:cost-complexity-upper-bound-C}}
  \paragraph{Step 4 for \explore{C}: Prove that if the events of Steps 2 and 3 do not happen, for \(\Lambda > O\left(Q \log \left( \frac{KM\sqrt{Q}}{\delta (\lambda\upbra{1})^2} \right)\right)\), the algorithm terminates with a probability at least \(O(1 - \delta/\Lambda^2)\).} Denote \(\bar{T} \coloneqq \ceil{\frac{\Lambda}{2\lambda\upbra{M}}}\)
  and two events \(E_1, E_2\) as follows,
  \[
    \begin{split}
      E_1 &\coloneqq \{\exists t \ge \bar{T}: \texttt{CROS}_t\},\\
      E_2 &\coloneqq \{\exists t\ge \bar{T}, k\in\mathcal{K}: (n_{k,t}\upbra{m^*_k} > 16n^*_{k,t})\cup (k\in\texttt{Mid}_t)\}.
    \end{split}
  \]
  We first upper bound the number of rounds after \(\bar{T}\) as follows,
  \begin{equation}\label{eq:bound-second-half-cost}
    \begin{split}
      \sum_{t\ge \bar{T}} \lambda\upbra{m_t} \1{\lnot \texttt{TERM}_t}
      & \overset{(a)}= \sum_{t\ge \bar{T}}  \lambda\upbra{m_t}\1{\lnot \texttt{TERM}_t \cap \lnot \texttt{CROS}_t} \\
      &\overset{(b)} \le  \sum_{t\ge \bar{T}} \lambda\upbra{m_t} \1{(\ell_t\in\texttt{Mid}_t) \cup (u_t\in\texttt{Mid}_t)}\\
      &\le \sum_{t\ge \bar{T}} \sum_{k\in\mathcal{K}}  \lambda\upbra{m_t}\1{((k = \ell_t) \cup (k = u_t))\cap (k\in\texttt{Mid}_t)}\\
      &\overset{(c)}\le \sum_{t\ge \bar{T}} \sum_{k\in\mathcal{K}} \lambda\upbra{m_t} \1{((k = \ell_t) \cup (k = u_t))\cap (n_{k,t}\upbra{m^*_k} \le 16n^*_{k,t})}\\
      &=  \sum_{k\in\mathcal{K}}\sum_{t\ge \bar{T}}  \lambda\upbra{m_t} \1{((k = \ell_t) \cup (k = u_t))\cap (n_{k,t}\upbra{m^*_k} \le 16n^*_{k,t})}\\
      &\le \sum_{k\in\mathcal{K}} \left( \sum_{\ell=1}^{m^*_k - 1} \frac{\lambda\upbra{\ell}\log (Lt^4/\delta)}{(\zeta\upbra{\ell})^2} + 16\lambda\upbra{m^*_k}n^*_{k,t} \right)\\
      & \le \sum_{k\in\mathcal{K}} \left( \sum_{\ell=1}^{m^*_k - 1} \frac{\lambda\upbra{\ell}}{(\zeta\upbra{\ell})^2} + \frac{16\lambda\upbra{m^*_k}}{\Delta_k^2}\right) \log \left( \frac{LT^4}{\delta} \right)\\
      &\le 4 \sum_{k\in\mathcal{K}} \left( \sum_{\ell=1}^{m^*_k - 1} \frac{\lambda\upbra{\ell}}{(\zeta\upbra{\ell})^2} + \frac{16\lambda\upbra{m^*_k}}{\Delta_k^2}\right) \log \left( \frac{LT}{\delta} \right)\\
      &\le 4 \sum_{k\in\mathcal{K}} \left( \sum_{\ell=1}^{m^*_k - 1} \frac{\lambda\upbra{\ell}}{(\zeta\upbra{\ell})^2} + \frac{16\lambda\upbra{m^*_k}}{\Delta_k^2}\right) \log \left( \frac{L\Lambda}{\lambda\upbra{1}\delta} \right)
    \end{split}
  \end{equation}
  where the equation (a) is due to \(\lnot E_1\),
  the inequality (b) is due to Step 1, and
  the inequality (c) is due to \(\lnot E_2\).

  Also notice that \[
    \Lambda = \sum_{t < \bar{T}} \lambda\upbra{m_t} + \sum_{t\ge \bar{T}} \lambda\upbra{m_t} \1{\lnot \texttt{TERM}_t} \le \frac{\Lambda}{2} + \sum_{t\ge \bar{T}} \lambda\upbra{m_t} \1{\lnot \texttt{TERM}_t},
  \]
  and, combining with \eqref{eq:bound-second-half-cost}, we have, \[
    \Lambda \le 8 \sum_{k\in\mathcal{K}} \left( \sum_{\ell=1}^{m^*_k - 1} \frac{\lambda\upbra{\ell}}{(\zeta\upbra{\ell})^2} + \frac{16\lambda\upbra{m^*_k}}{\Delta_k^2}\right) \log \left( \frac{L\Lambda}{\lambda\upbra{1}\delta} \right).
  \]
  Solving the above inequality concludes the cost complexity upper bound for \explore{C}.

  In the end of Step 4, we show that the LUCB algorithm fulfills the fixed confidence requirement.
  The probability that the algorithm does not terminate after spending \(\Lambda\) budget is upper bounded by \(\P(E_1 \cup E_2)\). Based on Steps 2 and 3, it can be upper bounded as follows,
  \[\begin{split}
      \P(E_1 \cup E_2) &\le \sum_{t> \bar{T}}\left( \frac{KM\delta}{Lt^3} + \frac{16\delta\sum_{k\in\mathcal{K}}\Delta_k^{-2}}{L t^4} \right) \le \Lambda \left( \frac{1}{\lambda\upbra{1}} - \frac{1}{2\lambda\upbra{M}} \right) \left( \frac{\delta}{(\Lambda / \lambda\upbra{1})^3} \right)\sum_{k\in\mathcal{K}}\Delta_k^{-2}\\
      & \le \frac{\delta }{\Lambda^2}\sum_{k\in\mathcal{K}}\Delta_k^{-2} \left( (\lambda\upbra{1})^2 - \frac{(\lambda\upbra{1})^3}{2\lambda\upbra{M}} \right)\\
      &\le \delta.
    \end{split}
  \]
\end{proof}

% !TeX root = ..\mf-mab.tex
\section{Detailed Theorems for Regret Minimization Case}\label{sec:regret-minimization-long}

We first present both the problem-independent (worst-case) and problem-dependent regret lower bounds in Section~\ref{subsec:regret-lower-bound} and then devise an elimination algorithm whose worst-case upper bounds match the worst-case lower bound up to some logarithmic factors and whose problem-dependent upper bound matches the problem-dependent lower bound in a class of \MFMAB in Section~\ref{subsec:elimination-algorithm}.
% In Appendix~\ref{subsec:half-regret-bound}, we propose a necessary and sufficient condition for \MFMAB under which the the problem-independent regret lower bound is reduced to \(\Omega(\Lambda^\frac{1}{2})\), and, based on this condition, we also devise an algorithm enjoying a matching regret upper bound.

\subsection{Regret Lower Bound}\label{subsec:regret-lower-bound}

We present the problem-independent regret lower bound in Theorem~\ref{thm:regret-lower-bound} and the problem-dependent regret lower bound in Theorem~\ref{thm:problem-dependent-regret-lower-bound}.
Both proofs are deferred to Appendix~\ref{subapp:regret-lower-bound} and~\ref{subapp:problem-dependent-regret-lower-bound} respectively.

\begin{theorem}[Problem-independent regret lower bound]\label{thm:regret-lower-bound}
  Given budget \(\Lambda\), the regret of \emph{\MFMAB} is lower bounded as follows,
  \[
    \inf_{\text{Algo}}\sup_{\mathcal{I}}\E \left[ R(\Lambda) \right]
    \ge \Omega \left( K^{1/3}\Lambda^{{2/3}} \right),
  \]
  where the \(\inf\) is over any algorithms,
  the \(\sup\) is over any possible \emph{\MFMAB} instances \(\mathcal{I}\).
\end{theorem}

\begin{theorem}[Problem-dependent lower bound]\label{thm:problem-dependent-regret-lower-bound}
  For any consistent policy that,
  after spending \(\Lambda\) budgets, fulfills that for any suboptimal arm \(k\) (with \(\Delta_k^{(M)}>0\)) and any \(a>0\),
  \(
  \E[ N_{k}^{(\forall m)} \left(\Lambda \right) ] = o(\Lambda^a),
  \)
  its regret is lower bounded by the following inequality,
  \begin{equation*}
    \label{eq:problem-dependent-regret-lower-bound-general}
    \liminf_{\Lambda\to\infty }\frac{\E\left[ R(\Lambda) \right]}{\log (\Lambda)} \ge  \sum_{k\in \mathcal{K}} \min_{m:\Delta_k^{(m)}>0} \left( \frac{\lambda\upbra{m}}{\lambda\upbra{1}} \mu_1\upbra{M} - \mu_k\upbra{M} \right)\frac{C}{(\Delta_k^{(m)})^2}.
  \end{equation*}
  % Especially, when constrained to the instances that \(\Delta_k\upbra{m} < 0\) for any arm \(k\) and fidelity \(m\neq M\),
  % the lower bound becomes,
  % \begin{equation}
  %     \label{eq:problem-dependent-regret-lower-bound-special}
  %     \liminf_{\Lambda\to\infty }\frac{\E\left[ R(\Lambda) \right]}{\log (\Lambda)} \ge  \sum_{k\in \mathcal{K}} \left( \frac{\lambda\upbra{M}}{\lambda\upbra{1}} \mu_1\upbra{M} - \mu_k\upbra{M} \right)\frac{C}{(\Delta_k^{(M)})^2}.
  % \end{equation}
\end{theorem}

The above two lower bound proofs utilize two different regret decomposition as follows,
\begin{align}
  R(\Lambda)
   & \overset{(a)}\ge \sum_{t=1}^N \left(
  \frac{\lambda^{(m_t)}-\lambda^{(1)}}{\lambda^{(1)}}\mu_1 + \Delta_{I_t}\upbra{M}
  \right)\nonumber
  \\
   & = \sum_{m=2}^M N_{\forall k}^{(m)}(\Lambda) \frac{\lambda^{(m)}-\lambda^{(1)}}{\lambda^{(1)}}\mu_1
  + \sum_{k\neq k_1^{(M)}} N_k^{(\forall m)}(\Lambda) \Delta_{k}\upbra{M};
  \label{eq:worst-case-regret-decomposition}
  \\
  R(\Lambda)
   & = \sum_{k\in\mathcal{K}} \sum_{m=1}^M N_k^{(m)}(\Lambda)\left( \frac{\lambda\upbra{m}}{\lambda\upbra{1}} \mu_1\upbra{M} - \mu_k\upbra{M} \right),
  \label{eq:problem-dependent-regret-decomposition}
\end{align}
where the inequality (a) is due to \(\Lambda > \sum_{t=1}^N \lambda^{(m_t)}\), the \(N_k^{(m)}(\Lambda)\) is the number of times that arm \(k\) is pulled in fidelity \(m\) after paying budget \(\Lambda\). The \(N\) with subscript \(_{\forall k}\) and superscript \(\upbra{\forall m}\) mean the pulling times of all arms and all fidelities respectively.
Due to the multi-fidelity feedback, both decompositions are different from classic bandits' regret decomposition~\citep[Lemma 4.5]{lattimore2020bandit}, and, therefore, we need to non-trivially extend the known approaches to our scenario.

To prove the problem-independent lower bound \(\Omega(K^{1/3}\Lambda^{2/3})\) in Theorem~\ref{thm:regret-lower-bound}, we utilize the decomposition in~\eqref{eq:worst-case-regret-decomposition}:
In the RHS, the first term increases whether one choose fidelity other than the lowest; this is a novel term.
The second term of RHS corresponds to the cost of pulling suboptimal arms which also appears in the classic MAB.
With this observation, one can construct instance pairs such that it is unavoidable to do exploration at higher fidelities and, therefore, the first term is not negligible.
Lastly, one needs to balance the magnitude of the above two terms case-by-case, which together bounds the regret  as \(\Omega \left( \Lambda^{{2/3}} \right)\).
To prove the problem-dependent lower bound in Theorem~\ref{thm:probdep-regret-upper-bound}, we utilize~\eqref{eq:problem-dependent-regret-decomposition} to decompose the regret to each arm and bound each of them separately.

\subsection{An Elimination Algorithm and Its Regret Upper Bound}\label{subsec:elimination-algorithm}

In this section, we propose an elimination algorithm for \MFMAB based on~\citet{auer2010ucb}.
This algorithm proceeds in phases \(p=0,1,\dots\) and maintains a candidate arm set \(\mathcal{C}_p\).
The set \(\mathcal{C}_p\) is initialized as the full arm set \(\mathcal{K}\) and the algorithm gradually eliminates arms from the set until there is only one arm remaining.
When the candidate arm set contains more than one arms, the algorithm explores arms with the highest fidelity \(M\),
and when the set \(\abs{\mathcal{C}_p}=1\), the algorithm exploits the singleton in the set with the lowest fidelity \(m=1\).
We present the detail in Algorithm~\ref{alg:elimination-for-rm}.

\begin{algorithm}[H]
  \caption{Elimination for \MFMAB}\label{alg:elimination-for-rm}
  \begin{algorithmic}[1]
    \State \textbf{Input: } full arm set \(\mathcal{K}\), budget \(\Lambda\), and parameter \(\varepsilon\)
    \State \textbf{Initialization: } phase \(p\gets 0\), candidate set \(\mathcal{C}_p\gets \mathcal{K}\)
    \While{\(p < \log_2 \frac{2}{\varepsilon}\) and \(\abs{\mathcal{C}_p} > 1\)}
    \State pull each arm \(k\in\mathcal{C}_p\) in highest fidelity \(M\) such that \(T_k\upbra{M} = \ceil*{2^{2p} \log \frac{\Lambda}{2^{2p}\lambda\upbra{M}}}\)
    \State Update reward means \(\hat{\mu}_{k,p}\upbra{M}\) for all arms \(k\in \mathcal{C}_p\)
    \State \(\mathcal{C}_{p+1} {\gets} \{k\in \mathcal{C}_p{:} \hat{\mu}_{k,p}\upbra{M} + 2^{-p + 1} {>} \max_{k'\in \mathcal{C}_p} \hat{\mu}_{k',p}\upbra{M}\}\)
    \Comment{Elimination}
    \State \(p\gets p+1\)
    \EndWhile
    \State Pull the remaining arms of \(\mathcal{C}_p\) in turn in fidelity \(m=1\) until the budget runs up
  \end{algorithmic}
\end{algorithm}

% \begin{remark}[Real-World Implication of Algorithm~\ref{alg:elimination-for-rm}]
%     There are real-world applications, e.g., the advertisement distribution problem in Section~\ref{subsec:application_regret_minimization}, where the explorations are conducted at the high fidelity (e.g., a large-scale user study) and the exploitations are conducted at the low fidelity (e.g., advertisement distribution via some platforms). This corroborates our algorithm design which explores at high fidelity and exploits at low fidelity.
%     On the other hand, the fact that our algorithm enjoys the tight regret performance comparing to regret lower bound (both problem-independent and -dependent, see next subsection) also implies that the approach used in real-world advertisement distribution is reasonable.
% \end{remark}

\subsubsection{Analysis Results}
We first present the problem-dependent regret upper bound of Algorithm~\ref{alg:elimination-for-rm} in Theorem~\ref{thm:problem-dependent-regret-upper-bound}.
Its full proof is deferred to Appendix~\ref{subapp:regret-upper-bound}.

\begin{theorem}[Problem-Dependent Regret Upper Bound]\label{thm:problem-dependent-regret-upper-bound}
  For any \(\varepsilon>0\). Algorithm~\ref{alg:elimination-for-rm}'s regret is upper bounded as follows,
  % \todo{if need, move the fnite time version to appendix}
  \begin{equation}
    \label{eq:problem-dependent-regret-upper-bound}
    \begin{split}
      & \E[R(\Lambda)] \le \max_{k:\Delta_k\upbra{M} \le \varepsilon} \frac{\Lambda}{\lambda\upbra{1}} \Delta_k\upbra{M}\\
      &+\!\!\!  \sum_{k:\Delta_k\upbra{M} > \varepsilon} \!\!\!
      \left(
      \left( \frac{\lambda\upbra{M}}{\lambda\upbra{1}} \mu_1\upbra{M} - \mu_k\upbra{M} \right)
      \left( \frac{16}{(\Delta_k\upbra{M})^2} \log\frac{\Lambda(\Delta_k\upbra{M})^2}{16\lambda\upbra{M}} + \frac{48}{(\Delta_k\upbra{M})^2} + 1 \right)
      + \frac{64}{\Delta_k\upbra{M}}
      \right)\\
      & + \!\!\!\sum_{k:\Delta_k\upbra{M} \le \varepsilon} \!\!\!\left(
      \left( \frac{\lambda\upbra{M}}{\lambda\upbra{1}}\mu_1\upbra{M} - \mu_k\upbra{M} \right)
      \left( \frac{16}{\varepsilon^2} \log\left( \frac{\Lambda \varepsilon^2}{16\lambda\upbra{M}}  \right) + \frac{32}{3\varepsilon^2} + 1\right) + \frac{64}{\varepsilon}.
      \right)
    \end{split}
  \end{equation}
  Especially, if letting \(\varepsilon\) go to zero and budget \(\Lambda\) go to infinity, the above upper bound becomes \begin{equation}
    \label{eq:asymptotical-problem-dependent-regret-upper-bound}
    \limsup_{\Lambda\to\infty }\frac{\E\left[ R(\Lambda) \right]}{\log (\Lambda)} \le  \sum_{k\in \mathcal{K}}
    \left( \frac{\lambda\upbra{M}}{\lambda\upbra{1}} \mu_1\upbra{M} - \mu_k\upbra{M} \right)
    \frac{16}{(\Delta_k\upbra{M})^2}.
  \end{equation}
\end{theorem}

% \begin{remark}[Class of \MFMAB Instances Where Problem-Dependent Regret Bound is Tight]
%     Denote \(\mathcal{M}_k\coloneqq \{m: \Delta_k\upbra{m} > 0\}\) as the set of fidelities that can be used to distinguish arm \(k\). When the instances we consider all fulfill the condition that \(\mathcal{M}_k = \{M\}\) for all arm \(k\in\mathcal{K}\), then
%     the problem-dependent regret lower bound Theorem~\ref{thm:problem-dependent-regret-lower-bound} is~\eqref{eq:problem-dependent-regret-lower-bound-special}.
%     It matches the asymptotical regret upper bound in~\eqref{eq:asymptotical-problem-dependent-regret-upper-bound}.
%     These instances with \(\mathcal{M}_k = \{M\}\) covers a vast number of real world applications. Because in practical, the highest fidelity \(M\) is often defined as the only fidelity where the suboptimal arms are clearly distinguished from the optimal arm. For example, in neural architecture search, the process of increasing the training sample size stops when one architecture performs much better than others.
%     Therefore, this class of \MFMAB instances is general, and many real-world applications belongs to the class.
% \end{remark}

Letting \(\varepsilon =  \left( {K\log\Lambda}/{\Lambda} \right)^{1/3}\) in~\eqref{eq:problem-dependent-regret-upper-bound} of Theorem~\ref{thm:problem-dependent-regret-upper-bound},
one can obtain a problem-independent regret upper bound of Algorithm~\ref{alg:elimination-for-rm} in Theorem~\ref{thm:regret-upper-bound} as follows.

\begin{theorem}[Regret Upper Bound for Algorithm~\ref{alg:elimination-for-rm}]\label{thm:regret-upper-bound}
  Letting \(\varepsilon =  \left( {K\log\Lambda}/{\Lambda} \right)^{1/3}\), Algorithm~\ref{alg:elimination-for-rm}'s regret is upper bounded as follows,
  \[
    \E[R(\Lambda)] \le O\left( K^{{1/3}} \Lambda^{{2/3}} (\log \Lambda)^{{1/3}} \right).
  \]
\end{theorem}

\section{Proofs for Regret Minimization Results}

\subsection{Proof of Theorem~\ref{thm:regret-lower-bound}}
\label{subapp:regret-lower-bound}
\paragraph{Step 1. Construct instances and upper bound KL-divergence}
\textcolor{blue}{Fix a policy \(\pi\).}
We construct two \MFMAB instances, each with \(K\) arms and \(M\) fidelities. For the pulling costs of different fidelities, we set \(\lambda^{(M)}\le 2\lambda^{(1)}\). For reward feedback, we assume all arms' reward distributions at any fidelity are Bernoulli, and denote \(\mu_k\upbra{m}(1), \mu_k\upbra{m}(2)\) as the reward means of these two instances.
Let \(\P_1 = \P_{\bm{\mu}(1),\pi}, \mathbb{E}_1 = \mathbb{E}_{\bm{\mu}(1),\pi}\)
and \(\P_2 = \P_{\bm{\mu}(2),\pi}, \mathbb{E}_2 = \mathbb{E}_{\bm{\mu}(2),\pi}\) be the probability measures and expectations
on the canonical \MFMAB model induced by \(\Lambda\)-budget interconnection of \(\pi\) and \(\bm{\mu}_1\,(\text{and }\bm{\mu}_2)\).
Denote \(k' = \argmin_{k\in\mathcal{K}} \E_1[N_k\upbra{m>1}(\Lambda)]\).

The only difference between both instance pair is in the arm \(k'\)'s reward mean for fidelities \(m>1\), so that instance \(1\)'s optimal arm is arm \(1\) and instance \(2\)'s optimal arm is arm \(k'\).
The detailed reward means are listed as follows,

\begin{center}
    \begin{tabular}{ccc}
        \hline
                                                  & \(m=1\)                                                                                             & \(m>1\)                                                                                                                                      \\\hline
        \((\mu_k\upbra{m}(1))_{k\in\mathcal{K}}\) & \multirow{2}{*}{\(\left( \frac{1}{2}+\Delta, \frac{1}{2}, \frac{1}{2}, \dots, \frac{1}{2}\right)\)} & \(\left( \frac{1}{2}+\Delta, \frac{1}{2}, \frac{1}{2}, \dots, \frac{1}{2}\right)\)                                                           \\
        \((\mu_k\upbra{m}(2))_{k\in\mathcal{K}}\) &                                                                                                     & \(\left( \frac{1}{2}+\Delta, \frac{1}{2},\dots, \frac{1}{2}, \textcolor{blue}{\frac{1}{2}+2\Delta}, \frac{1}{2}, \dots, \frac{1}{2}\right)\) \\
        \hline
    \end{tabular}
\end{center}

Denote the entry \((M_t, I_t, X_t)\) as a tuple of the pulled fidelity, pulled arm, and observed reward random variables at time \(t\), and \(\mathcal{H} \coloneqq (M_1, I_1, X_1; M_2, I_2, X_2; \dots ; M_N, I_N, X_N)\) as a sequence of applying policy \(\pi\). We note that \(N\) is also a random variable depending on the sequence of fidelities in pulling arms.
Next, we calculate the upper bound of the KL-divergence of the above two instances in this sequence.

\begin{equation}\label{eq:kl-div-upper-bound}
    \begin{split}
        \KL(\P_1, \P_2) &= \E_1\left[ \log\left( \frac{d\P_1}{d\P_2} \right) \right]\\
        &= \E_1\left[ \log\left( \frac{d\P_1}{d\P_2}
            (M_1, I_1, X_1; M_2,I_2,X_2;\dots;M_N,I_N,X_N)
            \right) \right]\\
        &= \E_1\left[\E_1\left[\left. \log\left( \frac{d\P_1}{d\P_2}
                (M_1, I_1, X_1; M_2,I_2,X_2;\dots;M_N,I_N,X_N)
                \right)\right\rvert N \right]  \right]\\
        &\overset{(a)}= \E_1\left[ \E_1\left[ \left. \sum_{t=1}^N \log\left( \frac{p_1(X_t\vert M_t, I_t)}{p_2(X_t\vert M_t, I_t)} \right) \right\rvert N \right]\right] \\
        &= \E_1\left[\sum_{t=1}^N \E_1\left[ \left.  \log\left( \frac{p_1(X_t\vert M_t, I_t)}{p_2(X_t\vert M_t, I_t)} \right) \right\rvert N \right]\right]\\
        &\overset{(b)}= \E_1\left[\sum_{t=1}^N \E_1\left[ \left.  \KL\left(P_1^{(M_t,I_t)}, P_2^{(M_t,I_t)}\right) \right\rvert N \right]\right]\\
        &= \sum_{(m,k)\in \mathcal{M}\times \mathcal{K}}\E_1\left[\sum_{t=1}^N \E_1\left[ \left. \1{M_t=m,I_t=k} \KL\left(P_1^{(M_t,I_t)}, P_2^{(M_t,I_t)}\right) \right\rvert N \right]\right]\\
        &= \sum_{(m,k)\in \mathcal{M}\times \mathcal{K}}\E_1\left[ \KL\left(P_1^{(m,k)}, P_{2}^{(m,k)}\right) \sum_{t=1}^N \E_1\left[ \left. \1{M_t=m,I_t=k}  \right\rvert N\right]\right]\\
        &= \sum_{(m,k)\in \mathcal{M}\times \mathcal{K}}\KL\left(P_1^{(m,k)}, P_2^{(m,k)}\right)\E_1\left[\E_1\left[ \left. \sum_{t=1}^N \1{M_t=m,I_t=k}  \right\rvert N \right]\right]\\
        &\overset{(c)}= \sum_{(m,k)\in \mathcal{M}\times \mathcal{K}}\KL\left(P_1^{(m,k)}, P_2^{(m,k)}\right)\E_1\left[T_k^{(m)}(\Lambda)\right]\\
        &=  \KL\left(P_1^{(m,2)}, P_2^{(m,2)}\right) \E_1\left[T_{k'}^{(m>1)}(\Lambda)\right]\\
        &\overset{(d)}\le \frac{1}{K}\cdot \E_1\left[
            T_{\forall k}^{(m>1)}(\Lambda)\right]
        \cdot \KL\left(
        \mathcal{B}\left(\frac{1}{2}\right),
        \mathcal{B}\left( \frac{1}{2}+2\Delta \right)
        \right)\\
        &\overset{(e)}\le  \E_1\left[
            T_{\forall k}^{(m>1)}(\Lambda)\right]
        \cdot \frac{9\Delta^2}{K},
    \end{split}
\end{equation}
where the equation (a) is due to \[
    \begin{split}
        &\quad \frac{d\P_1}{d(\rho\times\rho\times\lambda)^N}
        (m_1,k_1,x_1; m_2,k_2,x_2; \dots; m_N,k_N,x_N)\\
        &= p_{\bm{\mu}_1,\pi} (m_1,k_1,x_1; m_2,k_2,x_2; \dots; m_N,k_N,x_N)\\
        &=\prod_{t=1}^N\pi_t(m_t,k_t\vert m_1,k_1,x_1;\dots;m_{t-1},k_{t-1},x_{t-1}) p_{\bm{\mu}_1}(x_t\vert m_t,k_t),
    \end{split}
\]
the equation (b) is due to \[
    \begin{split}
        \E_1\left[ \left.  \log\left( \frac{p_1(X_t\vert M_t, I_t)}{p_2(X_t\vert M_t, I_t)} \right) \right\rvert N \right]
        &= \E_1\left[ \left.  \E_1\left[
                \log\left( \left. \frac{p_1(X_t\vert M_t, I_t)}{p_2(X_t\vert M_t, I_t)}\right) \right\rvert
                M_t, I_t
                \right]\right\rvert N \right]\\
        &= \E_1\left[ \left.  \KL\left(P_1^{(M_t,I_t)}, P_2^{(M_t,I_t)}\right) \right\rvert N \right],
    \end{split}
\]
the equation (c) is due to the tower property as well,
the inequality (d) is because arm \(k'\) is the arm with the smallest number of pulled with fidelity \(m > 1\),
and the inequality (e) is by calculating the KL-divergent between two Bernoulli distributions.

\paragraph{Step 2. Lower bound the regret}
We first note that the regret defined in \eqref{eq:lambda1-regret} for instance \(1\) can be decomposed as follows,
\[
    \begin{split}
        \mathbb{E}_1 [R(\Lambda)] &\overset{(a)}\ge \sum_{t=1}^N \left(
        \frac{\lambda^{(m_t)}-\lambda^{(1)}}{\lambda^{(1)}}\mu_1 + \Delta_{I_t}\upbra{M}
        \right)\\
        &=\sum_{m=2}^M T_{\forall k}^{(m)}(\Lambda) \frac{\lambda^{(m)}-\lambda^{(1)}}{\lambda^{(1)}}\mu_1
        + \sum_{k\neq k_1^{(M)}} T_k^{(\forall m)}(\Lambda) \Delta_k\upbra{M}
    \end{split}
\]
where inequality (a) is due to \(\Lambda > \sum_{t=1}^N \lambda^{(m_t)}\), the \(T_k^{(m)}(\Lambda)\) is the number of times that arm \(k\) is pulled in fidelity \(m\) (given total budget \(\Lambda\)).

Then, we lower bound the summation of the regrets under both instances,
\[
    \begin{split}
        &\quad \E_1[R(\Lambda)] + \E_2[R(\Lambda)] \\
        &=\E_1\left[
        \sum_{m=2}^M T_{\forall k}^{(m)}(\Lambda) \frac{\lambda^{(m)}-\lambda^{(1)}}{\lambda^{(1)}}\mu_1
        + \sum_{k\neq k_1^{(M)}} T_k^{(\forall m)}(\Lambda) \Delta_k\upbra{M}
        \right] \\
        &\quad+
        \E_2\left[
        \sum_{m=2}^M T_{\forall k}^{(m)}(\Lambda) \frac{\lambda^{(m)}-\lambda^{(1)}}{\lambda^{(1)}}\mu_1
        + \sum_{k\neq k_1^{(M)}} T_k^{(\forall m)}(\Lambda) \Delta_k\upbra{M}
        \right]\\
        &\ge
        \E_1\left[
        T_{\forall k}^{(m>1)}(\Lambda) \frac{\lambda^{(2)}-\lambda^{(1)}}{\lambda^{(1)}}\mu_1
        + \sum_{k\neq k_1^{(M)}} T_k^{(\forall m)}(\Lambda) \Delta_k\upbra{M}
        \right] \\
        &\quad+
        \E_2\left[
        T_{\forall k}^{(m>1)}(\Lambda) \frac{\lambda^{(2)}-\lambda^{(1)}}{\lambda^{(1)}}\mu_1
        + \sum_{k\neq k_1^{(M)}} T_k^{(\forall m)}(\Lambda) \Delta_k\upbra{M}
        \right]\\
        &\overset{(a)}\ge  \E_1\left[
            T_{\forall k}^{(m>1)}(\Lambda)\right] \frac{\lambda^{(2)}-\lambda^{(1)}}{\lambda^{(1)}}\mu_1
        \\
        & \quad +
        \Delta \min\left\{
        \frac{\Lambda}{\lambda^{(M)}} - \frac{\Lambda}{2\lambda^{(1)}},
        \frac{\Lambda}{2\lambda^{(1)}}
        \right\}\left( \P_1\left( T_1^{(\forall m)} \le \frac{\Lambda}{2\lambda^{(1)}} \right) +  \P_2\left( T_1^{(\forall m)} > \frac{\Lambda}{2\lambda^{(1)}} \right)\right)\\
        &\overset{(b)}\ge\E_1\left[
            T_{\forall k}^{(m>1)}(\Lambda)\right] \frac{\lambda^{(2)}-\lambda^{(1)}}{\lambda^{(1)}}\mu_1 +\frac{\Lambda\Delta}{2}\min\left\{
        \frac{1}{\lambda^{(M)}} - \frac{1}{2\lambda^{(1)}},
        \frac{1}{2\lambda^{(1)}}
        \right\}\exp(-\KL(\P_1, \P_2))\\
        &\overset{(c)}\ge \E_1\left[
            T_{\forall k}^{(m>1)}(\Lambda)\right] \frac{\lambda^{(2)}-\lambda^{(1)}}{\lambda^{(1)}}\mu_1 \\
        &\qquad +\frac{\Lambda\Delta}{2}\min\left\{
        \frac{1}{\lambda^{(M)}} - \frac{1}{2\lambda^{(1)}},
        \frac{1}{2\lambda^{(1)}}
        \right\}\exp\left(-2\Delta^2 K^{-1} \E_1\left[ T_{\forall k}^{(m>1)}(\Lambda) \right]\right),
    \end{split}
\]
where inequality (a) uses the \(\lambda^{(M)}\le 2\lambda^{(1)}\) condition,
inequality (b) is by Bretagnolle-Huber inequality~\citep[Theorem 14.2]{lattimore2020bandit},
and inequality (c) is by \eqref{eq:kl-div-upper-bound}.

\paragraph{Step 3. Obtain the \(\Omega(K^{\frac 1 3}\Lambda^{\frac{2}{3}})\) lower bound} We show that \(\E_1[R(\Lambda)] + \E_2[R(\Lambda)] \ge \Omega(K^{\frac 1 3}\Lambda^\frac{2}{3})\) for any possible quantity of \(\E_1\left[
    T_{\forall k}^{(m>1)}(\Lambda)\right]\) via categorized discussion as follows,
\begin{itemize}
    \item Case 1: If \(\E_1\left[T_{\forall k}^{(m>1)}(\Lambda)\right] = 0\), then the last formula becomes \(C\Lambda\Delta\).
          Letting \(\Delta\) be a constant (via \(\sup\)),
          we have a \(\Omega(\Lambda)\) lower bound, which means
          \(\Omega(K^{\frac 1 3}\Lambda^\frac{2}{3})\) is also valid.
    \item Case 2: If \(\E_1\left[
              T_{\forall k}^{(m>1)}(\Lambda)\right]  \ge  K^{\frac 1 3}\Lambda^{\frac{2}{3}}\), then we have the \(\Omega(K^{\frac 1 3}\Lambda^{\frac{2}{3}})\) lower bound from the first term.
    \item Case 3: If \(0< \E_1\left[
              T_{\forall k}^{(m>1)}(\Lambda)\right]  < K^{\frac 1 3}\Lambda^{\frac{2}{3}}\), we choose \(\Delta = K^{\frac 1 3}\Lambda^{-\frac{1}{3}}\) and also obtain the \(\Omega(K^{\frac 1 3} \Lambda^\frac{2}{3})\) lower bound.
\end{itemize}

\subsection{Proof of Theorem~\ref{thm:problem-dependent-regret-lower-bound}}
\label{subapp:problem-dependent-regret-lower-bound}

In this proof, we prove that for any arm \(k\in \mathcal{K}\), the total regret due to this arm \(k\) is lower bounded as follows, \[
    \liminf_{\Lambda\to\infty }\frac{\E\left[ R_k(\Lambda) \right]}{\log (\Lambda)} \ge  \min_{m:\Delta_k^{(m)}>0} \left( \frac{\lambda\upbra{m}}{\lambda\upbra{1}} \mu_1\upbra{M} - \mu_k\upbra{M} \right)\frac{C}{(\Delta_k^{(m)})^2}.
\]
Then, noticing that \(R(\Lambda) = \sum_{k\in\mathcal{M}} R_k(\Lambda)\) concludes the proof.

We construct instances \(\nu\) and \(\nu'\).
We set the reward distributions \(\nu=(\nu_k\upbra{m})_{(k,m)\in\mathcal{K}\times\mathcal{M}}\) as Bernoulli and the reward means fulfill \(\mu_1\upbra{M}>\mu_2\upbra{M}\ge \mu_3\upbra{M}\ge \dots \ge \mu_K\upbra{M},\) where \(\mu_k\upbra{m} = \mathbb{E}_{X\sim \nu_k\upbra{m}}[X]\).
We let \({\nu'}_k\upbra{m}\) be the same to \(\nu_k\upbra{m}\) for all \(k\) and \(m\), except for that an arm \(\ell\neq 1\).
We set arm \(\ell\)'s reward means  on fidelities \(m\in\mathcal{M}_k\) to be \({\nu'}_\ell\upbra{m} = \nu_1\upbra{M}-\zeta\upbra{m} + \epsilon\).
% So, in instance \(\nu'\), the optimal arm is \(\ell\) and its true reward mean \({\mu'}_\ell\upbra{M}\) is slightly greater than \(\mu_1\upbra{M}\).
% Given an instance \(\mathcal{I}\), for any fixed suboptimal arm \(k\), we consider another instance \(\mathcal{I}'\) whose reward distributions are the same as \(\mathcal{I}\) except for arm \(k\).
% Denote \(\tilde{m}_k\) as the first fidelity (in \(\mathcal{I}\)) such that \(\Delta_k\upbra{\tilde{m}_k} > 0\), i.e., \(\tilde{m}_k \coloneqq \min\{m\in \mathcal{M}: \Delta_k\upbra{m} > 0\}\).
% At instance \(\mathcal{I}'\), for any \(\epsilon > 0\), we set \({\mu'}_k\upbra{\tilde{m}_k}\in (\mu_{k_*}\upbra{M} - \zeta\upbra{\tilde{m}_k}, \mu_{k_*}\upbra{M} - \zeta\upbra{\tilde{m}_k} + \min_{m:\Delta_k\upbra{m} > 0} -\Delta_k\upbra{m})\)
% such that \(\KL(\mu_k\upbra{\tilde{m}_k}, {\mu'}_k\upbra{\tilde{m}_k})
% < (1+\epsilon) \KL(\mu_k\upbra{\tilde{m}_k}, \mu_{k_*}\upbra{M} - \zeta\upbra{\tilde{m}_k})\eqqcolon (1+\epsilon)\KL_k\upbra{\tilde{m}_k}\).
% For arm \(k\)'s reward means of other fidelities, we set \[
%     {\mu'}_k\upbra{m} = \begin{cases}
%         \mu_k\upbra{m}                                                          & \text{ for } m: \Delta_k\upbra{m} < 0 \\
%         {\mu'}_k\upbra{\tilde{m}_k} + \zeta\upbra{\tilde{m}_k} - \zeta\upbra{m} & \text{ for } m: \Delta_k\upbra{m} > 0
%     \end{cases}.
% \]
One can verify that the reward means of arm \(k\) under instance \(v'\) fulfill the condition that \(\abs{{\mu'}_k\upbra{m} - {\mu'}_k\upbra{M}} \le \zeta\upbra{m}\) for any fidelity \(m\).
Notice that \({\mu'}_k\upbra{M} > \mu_{k_*}\upbra{M}\).
Hence, under instance \(\mathcal{I}'\), the optimal arm is \(k\).
We denote \(\P, \E\) and \(\P', \E'\) as the probability measures and expectations for instances \(\mathcal{I}\) and \(\mathcal{I}'\) respectively.

% \xuchuang{update this instance such that the KL term becomes \(\Delta\) terms.}

Next, we employ the (extended) key inequality from~\citet{garivier2019explore} as follows, \[
    \sum_{k\in\mathcal{K}} \sum_{m\in\mathcal{M}} \E[N_k\upbra{m}(\Lambda)] \KL(v_k\upbra{m}, {v'}_k\upbra{m}) \ge \kl(\E[Z], \E'[Z]).
\]

Let \(Z = \lambda\upbra{M}N_\ell\upbra{\forall m}(\Lambda) / \Lambda\) in the above inequality, we have \begin{equation}
    \label{eq:constraint-of-problem-dependent-regret-lower-bound}
    \begin{split}
        & \quad  \sum_{m:\Delta_\ell\upbra{m} > 0} \E[N_k\upbra{m}(\Lambda)] \KL(v_\ell\upbra{m}, {v'}_\ell\upbra{m}) \\
        & \ge \kl\left( \frac{\E[N_\ell\upbra{\forall m}(\Lambda)]}{\Lambda}, \frac{\E'[N_\ell\upbra{\forall m}(\Lambda)]}{\Lambda} \right) \\
        & \overset{(a)}\ge \left( 1- \frac{\lambda\upbra{M}\E[N_\ell\upbra{\forall m}(\Lambda)]}{\Lambda} \right) \log \frac{\Lambda}{\Lambda - \lambda\upbra{M}\E'[N_\ell\upbra{\forall m}(\Lambda)]} - \log 2
    \end{split}
\end{equation}
where inequality (a) is due to that for all \((p,q) \in [0,1]^2\),\(\kl(p,q) \ge (1-p) \log (1/(1-q)) - \log 2\).

Notice that the regret attributed to any arm \(k\) can be decomposed and lower bounded as follows, \[
    R_\ell(\Lambda) \ge \sum_{m=1}^M N_\ell^{(m)}(\Lambda)\left( \frac{\lambda\upbra{m}}{\lambda\upbra{1}} \mu_1\upbra{M} - \mu_\ell\upbra{M} \right),
\]
with the constraint in~\eqref{eq:constraint-of-problem-dependent-regret-lower-bound}. Since this is a linear programming, we know its solution is reached at its vertex.
Therefore, we lower bound the regret as follows,
\begin{equation}
    \label{eq:middle-result-regret-lower-bound}
    \begin{split}
        R_\ell(\Lambda) &\ge  \min_{m:\Delta_k\upbra{m} > 0} \left( \frac{\lambda\upbra{m}}{\lambda\upbra{1}} \mu_1\upbra{M} - \mu_\ell\upbra{M} \right) \frac{1}{\KL(v_\ell\upbra{m}, {v'}_\ell\upbra{m})} \\
        &\qquad \times \left( \left( 1- \frac{\lambda\upbra{M}\E[N_\ell\upbra{\forall m}(\Lambda)]}{\Lambda} \right) \log \frac{\Lambda}{\Lambda - \lambda\upbra{M}\E'[N_\ell\upbra{\forall m}(\Lambda)]} - \log 2 \right),
    \end{split}
\end{equation}

Notice that the policy is consistent,
that is, \(\E[N_\ell\upbra{\forall m}(\Lambda)] = o(T^a)\)
and \(\E'[N_k\upbra{\forall m}(\Lambda)] = o(\Lambda^a)\) for any \(a\in (0,1]\) and any suboptimal arm \(k \neq \ell\).
We have \(\frac{\lambda\upbra{M}\E[N_\ell\upbra{\forall m}(\Lambda)]}{\Lambda} = o(1)\)
and
\[\Lambda - \lambda\upbra{M}\E'[N_\ell\upbra{\forall m}(\Lambda)]
    \le \lambda\upbra{M}\sum_{k\neq \ell} \E'[N_k\upbra{\forall m}(\Lambda)] = o(\Lambda^a).
\]

Dividing both sides of~\eqref{eq:middle-result-regret-lower-bound} by \(\Lambda\), and letting \(\Lambda\) go to infinity and \(a\) go to \(1\), we have \[
    \begin{split}
        \liminf_{\Lambda\to\infty} \frac{\E[R_\ell(\Lambda)]}{\Lambda} & \ge
        \min_{m:\Delta_\ell\upbra{m} > 0} \left( \frac{\lambda\upbra{m}}{\lambda\upbra{1}} \mu_1\upbra{M} - \mu_\ell\upbra{M} \right) \frac{1}{\KL(v_\ell\upbra{m}, {v'}_\ell\upbra{m})}\\
        & \ge
        \min_{m:\Delta_\ell\upbra{m} > 0} \left( \frac{\lambda\upbra{m}}{\lambda\upbra{1}} \mu_1\upbra{M} - \mu_\ell\upbra{M} \right) \frac{1}{\KL(v_\ell\upbra{m}, {v}_1\upbra{M}-\zeta\upbra{m}+\epsilon)}
    \end{split}
\]

To bound the optimal arm \(1\)'s sample cost, we use the same \(\nu\) as above and construct another instance \(\nu''\). The instance \(\nu''\)'s reward means are the same to \(\nu\) except for arm \(1\) whose reward means for fidelity \(m\in \mathcal{M}_1\) are set as \({\mu''}_1\upbra{m} = \mu_2\upbra{m}+\zeta\upbra{m}-\epsilon\). Then, with similar procedure as the above, we obtain
\[
    \liminf_{\Lambda\to\infty} \frac{\E[R_1(\Lambda)]}{\Lambda} \ge
    \min_{m:\Delta_k\upbra{m} > 0} \left( \frac{\lambda\upbra{m}}{\lambda\upbra{1}} \mu_1\upbra{M} - \mu_1\upbra{M} \right) \frac{1}{\KL(v_1\upbra{m}, {v}_2\upbra{m}+\zeta\upbra{m}-\epsilon)}
\]
% \[
%     \begin{split}
%         \min_{\mathbb{E}[N_1\upbra{m}],\forall m}  \sum_{m=1}^M \lambda\upbra{m} \E_\nu[N_1\upbra{m}(\sigma)]
%         \ge \min_{m\in\mathcal{M}_1} \frac{\lambda\upbra{m}}{(1+\varepsilon)\KL(\nu_1\upbra{m}, \nu_2\upbra{M}+\zeta\upbra{m})} \log \frac{1}{2.4\delta}.
%     \end{split}
% \]

\subsection{Proof of Theorem~\ref{thm:regret-upper-bound}}
\label{subapp:regret-upper-bound}

We first prove a problem dependent regret upper bound as follows and then convert this bound to the problem independent regret upper bound presented in Theorem~\ref{thm:regret-upper-bound}.

Denote \(\Delta_k\upbra{M} = \mu_1\upbra{M} - \mu_k\upbra{M}\), and, especially, \(\Delta_1\upbra{M}=0\).

\begin{theorem}[Problem-Dependent Regret Upper Bound]\label{thm:probdep-regret-upper-bound}
    For any \(\varepsilon>0\). Algorithm~\ref{alg:elimination-for-rm}'s regret is upper bounded as follows,
    \begin{equation}
        \label{eq:prob-dep-regret-upper-bound}
        \begin{split}
            &\E[R(\Lambda)]
            \le \sum_{k:\Delta_k\upbra{M} > \varepsilon}
            \left(
            \left( \frac{\lambda\upbra{M}}{\lambda\upbra{1}} \mu_1\upbra{M} - \mu_k\upbra{M} \right)
            \left( \frac{16}{(\Delta_k\upbra{M})^2} \log\frac{\Lambda(\Delta_k\upbra{M})^2}{16\lambda\upbra{M}} + \frac{48}{(\Delta_k\upbra{M})^2} + 1 \right)
            + \frac{64}{\Delta_k\upbra{M}}
            \right)\\
            & + \sum_{k:\Delta_k\upbra{M} \le \varepsilon} \left(
            \left( \frac{\lambda\upbra{M}}{\lambda\upbra{1}}\mu_1\upbra{M} - \mu_k\upbra{M} \right)
            \left( \frac{16}{\varepsilon^2} \log\left( \frac{\Lambda \varepsilon^2}{16\lambda\upbra{M}}  \right) + \frac{32}{3\varepsilon^2} + 1\right) + \frac{64}{\varepsilon}
            \right)
            + \max_{k:\Delta_k\upbra{M} \le \varepsilon} \frac{\Lambda}{\lambda\upbra{1}} \Delta_k\upbra{M}.
        \end{split}
    \end{equation}
\end{theorem}

\begin{proof}[Proof of Theorem~\ref{thm:probdep-regret-upper-bound}]
    By Hoeffding's inequality, we have
    \[
        \begin{split}
            \P\left( \hat{\mu}_{k,p}\upbra{M} > \mu_k\upbra{M} + 2^{-p} \right) &\le \exp\left( - \frac{2^{-2p}}{2\times \frac{1}{2\times 2^{2p} \log (\Lambda/2^{2p}\lambda\upbra{M})}} \right) = \frac{2^{2p}\lambda\upbra{M}}{\Lambda}\\
            \P\left( \hat{\mu}_{k,p}\upbra{M} < \mu_k\upbra{M} - 2^{-p} \right) &\le \exp\left( - \frac{2^{-2p}}{2\times \frac{1}{2\times 2^{2p} \log (\Lambda/2^{2p}\lambda\upbra{M})}} \right) = \frac{2^{2p}\lambda\upbra{M}}{\Lambda}.
        \end{split}
    \]
    That is, the empirical mean \(\hat{\mu}_k\upbra{M}\) is within the confidence interval \((\mu_k\upbra{M} - 2^{-p}, \mu_k\upbra{M} + 2^{-p})\) with high probability.

    Choose any \(\varepsilon > \frac{e}{\Lambda}\). Let \(\mathcal{K}'=\{k\in\mathcal{K}|\Delta_k\upbra{M} > \varepsilon\}\).
    Denote \(p_k \coloneqq \min\{p: 2^{-p} <\frac{\Delta_k\upbra{M}}{2}\}\). From \(p_k\)'s definition, we have the following inequality \begin{equation}\label{eq:phase-gap-ineq}
        2^{p_k} < \frac{4}{\Delta_k\upbra{M}} < 2^{p_k+1}.
    \end{equation}
    We also note that the cost of pulling a suboptimal arm \(k\in\mathcal{K}'\) at highest fidelity \(M\) is upper bounded as \(\frac{\lambda\upbra{M}}{\lambda\upbra{1}} \mu_1\upbra{M} - \mu_k\upbra{M}\), where the factor \(\frac{\lambda\upbra{M}}{\lambda\upbra{1}}\) is because the budget paying to pull an arm at fidelity \(M\) can be used to the the arm at fidelity \(1\) for fractional times.

    The rest of this proof consists of two steps. In the first step, we assume that all empirical means are in their corresponding confidence intervals at each phases, and show the algorithm can \emph{properly} eliminate all suboptimal arms in \(\mathcal{K}'\)---the arm is eliminated in or before the phase \(p_k\).
    In the second step, we upper bound the regret if there are any empirical estimates lying outside their corresponding confidence intervals.

    \textbf{Step 1.}
    If all arms' empirical means lie in confidence intervals. That is, any suboptimal arm \(k\) is eliminated in or before the phase \(p_k\). Because, if \(k\in\mathcal{C}_{p_k}\), we have \[\begin{split}
            \hat{\mu}_k\upbra{M} &\le \mu_k\upbra{M} + 2^{-p_k} = \mu_1\upbra{M} - \Delta_k\upbra{M} + 2^{-p_k} \\
            &\overset{(a)}\le \mu_1\upbra{M} - 4 \times 2^{-p_k-1} + 2^{-p_k} = \mu_1\upbra{M} - 2^{-p_k} < \hat{\mu}_1\upbra{M} \le \max_{k\in\mathcal{C}_{p_k}} \hat{\mu}_k\upbra{M},
        \end{split}
    \]
    where (a) is due to \eqref{eq:phase-gap-ineq}.
    That is, if this arm \(k\) haven't been eliminated before phase \(p_k\), it must be eliminated in this phase.
    Therefore, the total pulling times of this arm \(k\) at highest fidelity \(M\) is upper bounded as follows, \begin{equation}
        \label{eq:pulling-times-upper-bound}
        T_k\upbra{M} \le \ceil*{2^{2p_k} \log \frac{\Lambda}{2^{2p_k}\lambda\upbra{M}}} \overset{(a)}\le \frac{16}{(\Delta_k\upbra{M})^2} \log \left( \frac{(\Delta_k\upbra{M})^2 \Lambda}{16\lambda\upbra{M}} \right) + 1,
    \end{equation}
    where (a) is due to \eqref{eq:phase-gap-ineq}.

    We then handle the number of times of pulling arms \(k\) with \(\Delta_k\upbra{M} \le \varepsilon\) in fidelity \(M\).
    Although these arms' total pulling times, eliminated in or before phase \(p_k,\) are also upper bounded by \eqref{eq:pulling-times-upper-bound}, their corresponding phases \(p_k\) is greater than \(\log_2\frac{2}{\varepsilon}\) and, therefore, cannot be reached. So, these arms' (including the optimal arm \(1\)'s) total pulling times in fidelity \(M\) is upper bounded by
    \[
        \frac{16}{\varepsilon^2} \log \left( \frac{\Lambda\varepsilon^2}{16\lambda\upbra{M}} \right) + 1.
    \]

    Therefore, the cost due to pulling arms at fidelity \(M\) is upper bounded as follows,
    \begin{equation}
        \label{eq:regret-decomposition-1}
        \sum_{k\in \mathcal{K}}\left( \frac{\lambda\upbra{M}}{\lambda\upbra{1}} \mu_1\upbra{M} - \mu_k\upbra{M} \right)\left( \frac{16}{\left(\max\left\{\varepsilon,\Delta_k\upbra{M} \right\}\right)^2} \log \left( \frac{\left(\max\left\{\varepsilon,\Delta_k\upbra{M} \right\}\right)^2 \Lambda}{16\lambda\upbra{M}} \right) + 1 \right).
    \end{equation}

    After the elimination process, arms with \(\Delta_k\upbra{M} \le \varepsilon\) may remain in the candidate arm set \(\mathcal{C}_p\) and are exploited in turn at fidelity \(m=1\). As some of them are not the optimal arm, this additional cost can be upper bounded as follows,
    \begin{equation}
        \label{eq:regret-decomposition-2}
        \max_{k:\Delta_k\upbra{M} \le \varepsilon} \frac{\Lambda}{\lambda\upbra{1}} \Delta_k\upbra{M}.
    \end{equation}

    \textbf{Step 2.} There are two cases that some arms are eliminated improperly: for an suboptimal arm \(k\), either \begin{enumerate}
        \item[2.1] The suboptimal arm \(k\) is \emph{not} eliminated in (or before) the phase \(p_k\), and the optimal arm \(1\) is in \(\mathcal{C}_{p_k}\) in phase \(p_k\); or
        \item[2.2] The suboptimal arm \(k\) is eliminated in (or before) the phase \(p_k\), and the optimal arm \(1\) is \emph{not} in \(\mathcal{C}_{p_k}\) in phase \(p_k\).
    \end{enumerate}

    Case 2.1's happening means that arm \(k\) is not eliminated in or before phase \(p_k\), which can only happen when the arm's empirical mean \(\hat{\mu}_k\upbra{M}\) lies outside its corresponding confidence interval. This event in or before phase \(p_k\) is with a probability no greater than \(2\times\frac{2^{2p_k}\lambda\upbra{M}}{\Lambda}\). Since this event may happen to any suboptimal arm \(k\in \mathcal{K}'\), then the regret of this case is upper bounded by
    \begin{equation}
        \label{eq:regret-decomposition-3}
        \begin{split}
            \sum_{k\in \mathcal{K}'} \frac{2^{2p_k+1}\lambda\upbra{M}}{\Lambda} \frac{\Lambda}{\lambda\upbra{M}} \left( \frac{\lambda\upbra{M}}{\lambda\upbra{1}} \mu_1\upbra{M} - \mu_k\upbra{M} \right)
            & = \sum_{k\in \mathcal{K}'} 2^{2p_k+1}\left( \frac{\lambda\upbra{M}}{\lambda\upbra{1}} \mu_1\upbra{M} - \mu_k\upbra{M} \right)\\
            & \overset{(a)}\le \sum_{k\in \mathcal{K}'}  \frac{32}{(\Delta_k\upbra{M})^2} \left( \frac{\lambda\upbra{M}}{\lambda\upbra{1}} \mu_1\upbra{M} - \mu_k\upbra{M} \right),
        \end{split}
    \end{equation}
    where (a) is due to \eqref{eq:phase-gap-ineq}.

    If Case 2.2 happens, the optimal arm \(1\) is not in the candidate arm set \(\mathcal{C}_{p_k}\) in phase \(p_k\). We denote \(p_1\) as the phase that the optimal arm \(1\) is eliminated,
    and it is at this phase that some arms' empirical means lie outside their confidence interval so that this mis-elimination happens.
    The probability of this event is upper bounded by \(2\times\frac{2^{2p_1}\lambda\upbra{M}}{\Lambda}\).
    We assume that arms \(k\) with \(p_k < p_1\) are eliminated in or before phase \(p_i\) properly; otherwise, the regret is counted in Case 2.1. Therefore, the optimal arm \(1\) eliminated in phase \(p_1\) should be eliminated by an arm \(k\) with \(p_k \ge p_1\).
    Consequently, the maximal per time slot regret in Case 2.2 is among arms with \(p_k \ge p_1\).
    Denote \(p_\varepsilon \coloneqq \min\{p|2^{-p}<\frac{\varepsilon}{2}\}\). We bound the cost of Case 2.2 as follows,
    \begin{equation}
        \label{eq:regret-decomposition-4}
        \begin{split}
            &\quad \sum_{p_1=0}^{\max_{i\in \mathcal{K}'} p_i} \sum_{k> 1: p_k \ge p_1} \frac{2^{2p_1+1}\lambda\upbra{M}}{\Lambda} \frac{\Lambda}{\lambda\upbra{M}} \cdot \max_{k'> 1: p_{k'} \ge p_1} \left( \frac{\lambda\upbra{M}}{\lambda\upbra{1}} \mu_1\upbra{M} - \mu_{k'}\upbra{M} \right)\\
            &=\sum_{p_1=0}^{\max_{i\in \mathcal{K}'} p_i} \sum_{k> 1: p_k \ge p_1} 2^{2p_1+1} \cdot \left( \frac{\lambda\upbra{M}}{\lambda\upbra{1}}\mu_1\upbra{M} - \mu_1\upbra{M} + \max_{k'> 1: p_{k'}\ge p_1} \Delta_{k'}\upbra{M} \right)\\
            &\overset{(a)}\le \sum_{p_1=0}^{\max_{i\in \mathcal{K}'} p_i} \sum_{k> 1: p_k \ge p_1} 2^{2p_1+1} \cdot \left( \frac{\lambda\upbra{M}}{\lambda\upbra{1}}\mu_1\upbra{M} - \mu_1\upbra{M} + 4 \times 2^{-p_1} \right)\\
            &\overset{(b)}\le \sum_{k > 1}\sum_{p_1=0}^{\min\{p_k, p_\varepsilon\}} 2^{2p_1+1} \cdot \left( \frac{\lambda\upbra{M}}{\lambda\upbra{1}}\mu_1\upbra{M} - \mu_1\upbra{M} + 4 \times 2^{-p_1} \right) \\
            &\le \sum_{k>1} \left(
            \left( \frac{\lambda\upbra{M}}{\lambda\upbra{1}}\mu_1\upbra{M} - \mu_1\upbra{M} \right) \sum_{p_1=0}^{\min\{p_k, p_\varepsilon\}} 2^{2p_1 + 1} + \sum_{p_1=0}^{\min\{p_k, p_\varepsilon\}} 2^{p_1 + 3}
            \right)\\
            &\le \sum_{k>1} \left(
            \left( \frac{\lambda\upbra{M}}{\lambda\upbra{1}}\mu_1\upbra{M} - \mu_1\upbra{M} \right)
            \frac{2^{2\min\{p_k, p_\varepsilon\}+1}}{3} + 2^{\min\{p_k, p_\varepsilon\} + 4}
            \right)\\
            &\le \sum_{k\in\mathcal{K}'} \left(
            \left( \frac{\lambda\upbra{M}}{\lambda\upbra{1}}\mu_1\upbra{M} - \mu_1\upbra{M} \right)
            \frac{2^{2p_k+1}}{3} + 2^{p_k + 4}
            \right) + \sum_{k>1, k\not\in \mathcal{K}'} \left(
            \left( \frac{\lambda\upbra{M}}{\lambda\upbra{1}}\mu_1\upbra{M} - \mu_1\upbra{M} \right)
            \frac{2^{2p_\varepsilon+1}}{3} + 2^{p_\varepsilon + 4}
            \right)\\
            & \overset{(c)}\le \sum_{k\in\mathcal{K}'} \left(
            \left( \frac{\lambda\upbra{M}}{\lambda\upbra{1}}\mu_1\upbra{M} - \mu_1\upbra{M} \right)
            \frac{32}{3(\Delta_k\upbra{M})^2} + \frac{64}{\Delta_k\upbra{M}}
            \right) + \sum_{k>1, k\not\in \mathcal{K}'} \left(
            \left( \frac{\lambda\upbra{M}}{\lambda\upbra{1}}\mu_1\upbra{M} - \mu_1\upbra{M} \right)
            \frac{32}{3\varepsilon^2} + \frac{64}{\varepsilon}
            \right)\\
            & \le \sum_{k\in\mathcal{K}'} \left(
            \left( \frac{\lambda\upbra{M}}{\lambda\upbra{1}}\mu_1\upbra{M} - \mu_k\upbra{M} \right)
            \frac{32}{3(\Delta_k\upbra{M})^2} + \frac{64}{\Delta_k\upbra{M}}
            \right) + \sum_{k>1, k\not\in \mathcal{K}'} \left(
            \left( \frac{\lambda\upbra{M}}{\lambda\upbra{1}}\mu_1\upbra{M} - \mu_k\upbra{M} \right)
            \frac{32}{3\varepsilon^2} + \frac{64}{\varepsilon}
            \right)
        \end{split}
    \end{equation}
    where (a) and (c) are due to \eqref{eq:phase-gap-ineq}, and (b) is due to the property of swapping two summations.

    Summing up the costs in \eqref{eq:regret-decomposition-1}, \eqref{eq:regret-decomposition-2}, \eqref{eq:regret-decomposition-3}, and \eqref{eq:regret-decomposition-4} concludes the proof.
\end{proof}

Next, we derive the problem-independent regret bound from \eqref{eq:prob-dep-regret-upper-bound}.
When \(\Lambda\) is large, the \(O(\log \Lambda)\) and \(O(\Lambda)\) terms dominate other terms in \eqref{eq:prob-dep-regret-upper-bound}.
For any given \(\varepsilon\), if \(\Lambda\) is large enough, we always have \(\varepsilon > 4\sqrt{\frac{\lambda\upbra{M} e}{\Lambda}}\), which, with some calculus, guarantees that \(\frac{16}{(\Delta_k\upbra{M})^2} \log\frac{\Lambda(\Delta_k\upbra{M})^2}{16\lambda\upbra{M}} < \frac{16}{\varepsilon^2} \log\frac{\Lambda\varepsilon^2}{16\lambda\upbra{M}}\) for all \(\Delta_k\upbra{M} > \varepsilon\).
Therefore, we can scale all logarithmic arm pulling times as \(\frac{16}{\varepsilon^2} \log\frac{\Lambda\varepsilon^2}{16\lambda\upbra{M}}\),
and upper bound the pulling cost \(\left( \frac{\lambda\upbra{M}}{\lambda\upbra{1}}\mu_1\upbra{M} - \mu_k\upbra{M} \right)\) by \(\frac{\lambda\upbra{M}}{\lambda\upbra{1}}\mu_1\upbra{M}\). We derive the problem-independent regret upper bound as follows,
\[
    \begin{split}
        \mathbb{E}\left[ R(\Lambda) \right] & \le
        \frac{K\mu_1\upbra{M}\cdot \lambda\upbra{M}}{\lambda\upbra{1}}\frac{16}{\varepsilon^2} \log\frac{\Lambda\varepsilon^2}{16\lambda\upbra{M}}
        + \frac{\Lambda}{\lambda\upbra{1}}\varepsilon\\
        &\le
        \frac{K\mu_1\upbra{M}\cdot \lambda\upbra{M}}{\lambda\upbra{1}}\frac{16}{\varepsilon^2} \log\frac{\Lambda}{16\lambda\upbra{M}}
        + \frac{\Lambda}{\lambda\upbra{1}}\varepsilon\\
        &\overset{(a)}\le 2 \left( \frac{16K\mu_1\upbra{M}\lambda\upbra{M}}{\lambda\upbra{1}} \log \frac{\Lambda}{16\lambda\upbra{M}} \right)^{\frac{1}{3}} \left( \frac{\Lambda}{\lambda\upbra{1}} \right)^{\frac{2}{3}},
    \end{split}
\]
where the equation of (a) holds when \(\varepsilon = \left( \frac{K\mu_1\upbra{M}\log(\Lambda/16\lambda\upbra{M})}{(\Lambda/16\lambda\upbra{M})} \right)^\frac{1}{3}.\)

\end{document}